\documentclass{article}

\usepackage{PRIMEarxiv}

\usepackage[utf8]{inputenc} 
\usepackage[T1]{fontenc}    
\usepackage{xcolor}         
\usepackage{hyperref}       
\usepackage{url}            
\usepackage{amsfonts}       
\usepackage{nicefrac}       
\usepackage{microtype}      
\usepackage{fancyhdr}       
\usepackage{natbib}         

\usepackage[inkscapelatex=false]{svg}
\usepackage{tabularx}
\usepackage{amssymb}
\usepackage{pifont}
\usepackage{booktabs}
\usepackage[most]{tcolorbox}
\usepackage{wrapfig}
\usepackage{tabularx}
\usepackage[most]{tcolorbox}
\usepackage{graphicx}
\usepackage{tabularx}
\usepackage{algorithm}
\usepackage{algorithmicx}
\usepackage{algcompatible}
\usepackage{algcompatible}
\usepackage{array}
\usepackage{enumitem}
\usepackage[most]{tcolorbox}

\usepackage{amsmath}
\usetikzlibrary{arrows.meta, positioning}

\definecolor{elegantred}{RGB}{150,40,40}

\usepackage{enumitem}
\usepackage{tikz}

\newcommand*\circled[1]{%
  \tikz[baseline=(char.base)]{
    \node[shape=circle, draw, inner sep=1pt, minimum size=0.7em] (char) {#1};
  }%
}

\newcommand{\indep}{\mathrel{\perp\!\!\!\perp}}

\usepackage{caption}

\definecolor{memfill}{RGB}{237,244,252}   
\definecolor{memline}{RGB}{ 60,105,160}
\definecolor{cffill} {RGB}{238,247,240}   
\definecolor{cfline} {RGB}{ 56,125, 82}
\definecolor{spurfill}{RGB}{255,226,140}  
\definecolor{spurtext}{RGB}{140, 60,  0}

\usepackage{hyperref}
\usepackage{url}

\usepackage{amsthm}
\usepackage{bbm}

\newtheorem{theorem}{Theorem}
\newtheorem{proposition}{Proposition}
\newtheorem{lemma}{Lemma}
\newtheorem{corollary}{Corollary}

\theoremstyle{definition}

\newtheorem{assumption}{Assumption}

\theoremstyle{remark}

\usepackage{minitoc}

\title{ \smash{\raisebox{-0.25\height}{\includegraphics[height=1.5em]{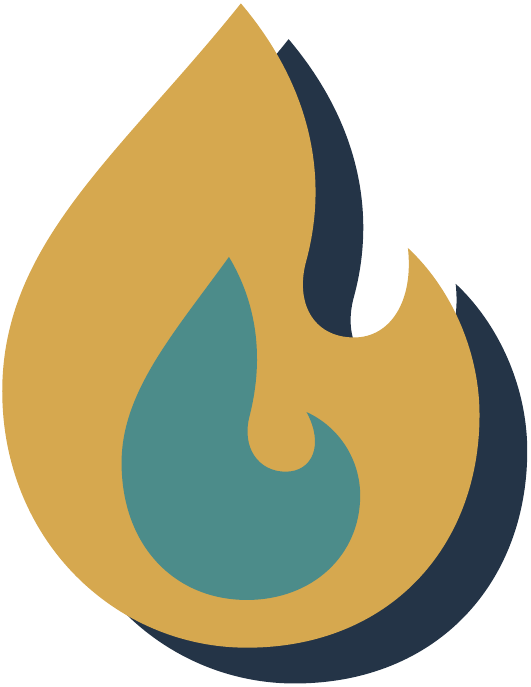}}}%
      \ Beyond the Shadows of Plato's Cave:\\Evaluating False Memory in Autonomous Agents via Counterfactual Reasoning
}

\author{
  \textbf{Quan M. Tran}$^{1}$, 
  \textbf{Zhuo Huang}$^{2}$,
  \textbf{Zhen Fang}$^{3}$,
  \textbf{Jing Zhang}$^{4}$,\\
  \textbf{Mingming Gong}$^5$,
  \textbf{Tongliang Liu}$^{1}$\\[1ex]
  \small{$^1$Sydney AI Centre, The University of Sydney;}
  \small{$^2$Australian Institute for Machine Learning, Adelaide University;}\\
  \small{$^3$Australian Artificial Intelligence Institute, University of Technology Sydney;}
  \small{$^4$Wuhan University;}\\
  \small{$^5$School of Mathematics and Statistics, The University of Melbourne}\\
}

\begin{document}
\doparttoc
\faketableofcontents

\maketitle

\begin{abstract}
Autonomous agents increasingly rely on memory to generalize beyond their training environments.
However, agents are bounded by what they have seen and believed, and leveraging such memories in unseen environments can introduce biases into their internal beliefs. We formalize this phenomenon as \textit{false memory}, which can arise from spurious correlations, environment shifts, and knowledge conflicts.
Despite its importance, false memory is difficult to evaluate because it stems from agent internal beliefs and is easily confounded with ordinary generalization failures.
Therefore, we propose FAME, a training-free framework that evaluates false memory through the evolution of agent beliefs under counterfactual reasoning.
Specifically, counterfactual scenarios reveal how beliefs change as the latent concept of memory shifts under hypothetical interventions; thus, measuring the resulting concept drift provides a signal for distinguishing faithful versus false memory.
Such concepts can be estimated from agent hidden states before answer generation, avoiding the need for reward design or answer sampling.
Empirical experiments reveal that simply monitoring answers often fails to detect false memory, while FAME achieves AUROCs of 76.2\% -- 96.7\% across false-memory settings, and outperforms the best baseline by 3.4\% -- 23.3\% across realistic benchmarks, spanning math reasoning (GSM-Symbolic), code generation (GitChameleon), and complex reasoning (BigBench-Hard).
We further release corresponding counterfactual templates and facilitate future research on false memory.
\end{abstract}

\section{Introduction}
Agents acquire broad capabilities through large-scale training~\citep{brown2020language, kaplan2020scaling, bommasani2021opportunities}, but these capabilities are largely fixed once training is finished~\citep {shinn2023reflexion, niu2022efficient, chen2026test}. 
In a fast-changing world, agents must continually adapt to new problems, environments, and evolving knowledge that cannot be fully anticipated during training~\citep{abel2023definition, lu2020reset}. 
This motivates autonomous agents to accumulate experience through self-improvement ~\citep{shinn2023reflexion, madaan2023self}, forming memories that retain successful experiences~\citep{wang2023voyager,yang2024buffer,zhao2024expel}, distill reusable skills\citep{mi2026skill, wang2024agent, yao2026can}, and support efficient adaptation to future problems~\citep{finn2017model, tran2026bifrost}. Memory therefore plays a critical role in extending agent-acquired knowledge to unseen environments.

However, agents with memory remain bounded by what they have seen and believed. Memories can encode beliefs entangled with specific environments in which agents were trained or practiced~\citep{cobbe2019quantifying, liao2026belief}, biasing internal beliefs when applied to unseen environments~\citep{hu2026evaluating}. 
We formalize this problem as \textit{false memory}, which can arise from spurious correlation~\citep{du2023shortcut, creager2021environment}, environment shift~\citep{amodei2016concrete}, and knowledge conflict~\citep{xu2024knowledge, zhao2026understanding, misra2026gitchameleon}. 
For instance, a personalized agent may recognize ``X'' as a mathematical notation rather than a social media network because its historical conversations were dominated by academic questions. 
An autonomous driving system specialized in the US may fail in the UK due to differences such as driving sides and measurement systems, while knowledge conflict may arise when an agent encounters a new policy or regulation that contradicts its prior knowledge. Such failures cause agents to retain or apply beliefs that are no longer valid, giving rise to false memory, analogous to Plato's allegory of the cave.

Despite its importance, false memory is challenging and costly to evaluate. Existing works largely conflate false memory with generalization failure and evaluate agents through observable answers~\citep{cobbe2019quantifying, kirk2021survey, zhong2024memorybank, maharana2024evaluating}.
Other approaches introduce reward functions or costly fine-tuning to construct evaluation data, requiring substantial human expertise and supervision~\citep{le2026sage, hou2024wikicontradict, su2024texttt}.
However, false memory originates from internal beliefs rather than observable outputs~\citep{orgad2025llms, azaria2023internal, kadavath2022language}, making answer-based evaluation and reward calibration inadequate, particularly when human expertise and labeled data are limited~\citep{li2024survey, cheang2026llms}. 
Reliable evaluation of false memory is therefore critical for developing effective adaptation strategies.

\begin{wrapfigure}{r}{0.45\textwidth}
    \vspace{-6mm}
    \includegraphics[width=\linewidth]{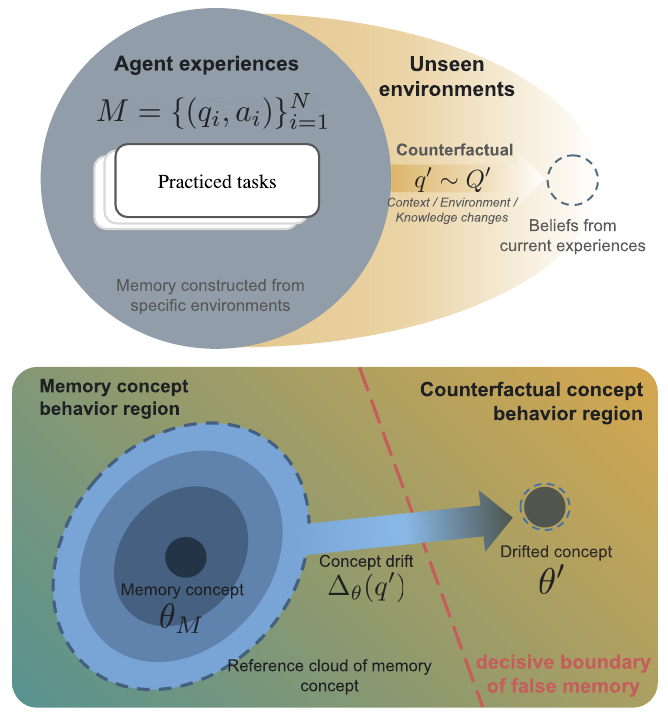}
    \vspace{-7mm}
    \caption{\small Illustration of FAME: It evaluates false memory through counterfactual reasoning by varying memory factors to observe concept drift. The counterfactual query and memory jointly establish boundaries for identifying false memory.}
    \label{fig:main}
    \vspace{-6mm}
\end{wrapfigure}
We propose FAME (Fig.~\ref{fig:main}), a training-free framework that evaluates false memory by tracking the evolution of an agent internal beliefs under counterfactual reasoning. 
Particularly, counterfactual scenarios reveal how the beliefs induced by memory respond to hypothetical interventions. 
Memory induces a latent concept that captures the agent belief~\citep{xie2021explanation, panwar2024context}, which can be approximated from hidden states~\citep{hendel2023context, todd2024function}.
When the agent encounters a counterfactual scenario, this concept may drift in the hidden-state representation space. 
Such drift can reflect appropriate adaptation or movement toward a region associated with false memory. 
FAME therefore measures whether the concept drift aligns with the decisive region by inducing counterfactual scenarios, providing a signal of whether the underlying memory remains faithful or has become false. 

For example, an autonomous driving system that practices in the US and stores its experiences as memory may fail when applied in the UK. FAME evaluates this through counterfactual reasoning: given the same memory, how does the agent-induced concept change when asked to drive in the UK? When robustness to environmental shifts is expected, the concept should remain within the memory reference region; thus, crossing this region indicates false memory. Moreover, consider a counterexample: if memory associates “X” with academic discussions, asking about “X” in a social-media context should induce a corresponding concept shift; thus, failure to adapt indicates false memory. When adaptation is expected, a faithful memory should shift toward the new context.

FAME is agnostic to answer generation and broadly applicable to realistic false-memory scenarios. 
When resolving a target query given a memory, the agent first induces a latent concept from memory demonstrations before generating the answer~\citep{xie2021explanation, hendel2023context, todd2024function}. 
We leverage this concept to assess false memory directly, avoiding answer generation and costly reward calibration. 
Moreover, FAME can be leveraged to evaluate empirical false-memory observations, i.e., spurious correlation, environment shift, and knowledge conflict, which can be characterized through counterfactual reasoning.

Our experiments demonstrate that false memory can be identified through concept drift, with statistically significant correlations (Pearson $r=0.48$–$0.98$, $p<0.001$) across three false-memory categories, i.e., spurious correlation, environment shift, and knowledge conflict, with AUROCs of 76.2\%–96.7\%. We further show that false memory is not necessarily observable from surface outputs, and evaluate FAME on realistic settings spanning math reasoning (GSM-Symbolic), code generation (GitChameleon), and complex reasoning (BigBench-Hard), where it outperforms the best baseline by 3.4\%–23.3\% AUROC. Finally, we release counterfactual templates to facilitate future research on false memory.

\vspace{-2mm}
In summary, our contributions are:
\begin{itemize}
    \vspace{-2mm}
    \item \textbf{We highlight false memory as a critical problem for autonomous agents (Sec.~\ref{sec:problem_setup}) and propose FAME (Sec.~\ref{sec:proposed_method})}, a training-free framework that evaluates false memory through counterfactual reasoning and concept drift without answer generation or reward calibration.
    \vspace{-2mm}
    \item \textbf{We provide a taxonomy (Sec.~\ref{sec:taxonomy_fm}) and theoretical foundation for false memory}. The taxonomy covers spurious correlation, environment shift, and knowledge conflict. The theoretical analysis establishes the identifiability of false-memory evaluation via counterfactual reasoning, and concept-drift realization via agent hidden states.
    \vspace{-2mm}
    \item \textbf{We empirically demonstrate the effectiveness of FAME (Sec.~\ref{sec:emp_eval},~\ref{sec:real_eval})} across diverse false-memory settings and release a suite of counterfactual templates to facilitate future research.
\end{itemize}

\section{Problem Setup}
\label{sec:problem_setup}


An autonomous agent stores its experiences in a memory space $\mathcal{M}$. When solving a new task, it retrieves a memory set $M \sim \mathcal{M}$, where $M=\{(q_i,a_i)\}_{i=1}^{N}$ consists of demonstrations with task queries and corresponding answers. These demonstrations induce a latent concept $\theta_M=f(M)$ that maximizes the expected reward on the memory tasks, $\theta_M=\arg\max_{\theta}\mathbb{E}_{(q,a)\sim M}[r(\theta\mid q,a)]$,
where $r$ denotes the reward function of the memory tasks.

A target query $\hat q$ may fall outside the agent past experiences, making the memory-induced concept $\theta_M$ no longer appropriate. We refer to this failure as \textit{false memory}, where applying $\theta_M$ to the target task yields substantially lower reward,
\vspace{0mm}
\begin{equation}
    \hat r(\theta_M\mid\hat q)
    \ll
    \mathbb{E}_{(q,a)\sim M}[r(\theta\mid q,a)],
\vspace{0mm}
\end{equation}
where $\hat r$ denotes the reward function of the target task.

However, this reward-based definition requires designing reward functions and generating answers for both memory and target tasks, which is costly. Moreover, false memory concerns the agent internal beliefs and therefore cannot always be identified from observable answers. 

We thus seek a training-free, answer-free measure based directly on the change in the memory-induced concept. Formally, let $Q'$ be the set of counterfactual queries that could induce false memory of $M$. False-memory evaluation measures the alignment of concept drift caused by each intervention of counterfactual query $q' \sim Q'$ with its admissible region distinguishing faithful vs. false memory,
\begin{equation}
    \mathcal{F}(M) = 
    \mathbb{E}_{q'\sim Q'}
    \left[
    \left\langle
    \Delta\theta(q'),\mathcal{A}(q')
    \right\rangle
    \right],
\end{equation}
where $\Delta_\theta(q')$ measures concept drift by counterfactual reasoning with an intervention on $q'$, and $\mathcal{A}(q')$ is its corresponding admissible region such that projecting the concept drift into it identifies false memory.

We will show how FAME realizes false-memory evaluation via counterfactual reasoning (Sec.~\ref{sec:proposed_method}) and how counterfactual queries $Q'$ reflect false-memory observations in practice (Sec.~\ref{sec:taxonomy_fm}).

\section{Proposed Method}
\label{sec:proposed_method}

\subsection{Behavioral Change of Concept via Counterfactual}
We demonstrate that counterfactual inference on different hypothetical scenarios provides a mechanism to measure concept drift for false-memory evaluation.

\paragraph{Structural Causal Model of belief update.}
\begin{equation}
\begin{aligned}
M = f_M(U_M),\quad
\Theta_M = f_{\Theta_M}(M, U_{\Theta_M}),\quad
Q = f_Q(U_Q),\quad
\Theta' = f'(Q, \Theta_M, U'),
\end{aligned}
\end{equation}
where $U=(U_M, U_{\Theta_M}, U_Q, U')$ is the set of exogenous variables, and $M$, $\Theta_M$, $Q$, and $\Theta'$ are the endogenous variables corresponding to the memory, memory concept, given task query, and updated concept, respectively. The SCM consists of the deterministic functions $f_M$, $f_{\Theta_M}$, $f_Q$, and $f'$.

\newsavebox{\causalbox}
\savebox{\causalbox}{%
  \begin{tikzpicture}[
      node distance = 0.3cm and 0.3cm,
      var/.style = {draw, circle, minimum size=0.7cm, inner sep=0pt, font=\small},
      arr/.style = {-{Stealth[length=1.2mm]}, thick}
  ]
  \node[var] (M) {$M$};
  \node[var, below=of M]      (thetaM) {$\Theta_M$};
  \node[var, below=of thetaM] (theta)  {$\Theta'$};
  \node[var, left=of theta]   (Q)      {$Q$};
  \node[var, below=of theta]  (A)      {$A$};
  \draw[arr] (M)      -- (thetaM);
  \draw[arr] (thetaM) -- (theta);
  \draw[arr] (Q)      -- (theta);
  \draw[arr] (theta)  -- (A);
  \end{tikzpicture}%
}
\begin{wrapfigure}{r}{0.16\textwidth}
\vspace{-4mm}
  \centering
  \vspace{-\intextsep}
  \begin{tikzpicture}[
      node distance = 0.3cm and 0.3cm,
      var/.style = {draw, circle, minimum size=0.7cm, inner sep=0pt, font=\small},
      arr/.style = {-{Stealth[length=1.2mm]}, thick, shorten >=0.5pt, shorten <=0.5pt}
  ]

  \node[var] (M) {$M$};
  \node[var, below=of M]      (thetaM) {$\Theta_M$};
  \node[var, below=of thetaM] (theta)  {$\Theta'$};
  \node[var, left=of theta]   (Q)      {$Q$};

  \draw[arr] (M)      -- (thetaM);
  \draw[arr] (thetaM) -- (theta);
  \draw[arr] (Q)      -- (theta);

  \end{tikzpicture}
  
  \vspace{-2mm}
  \caption{\small Causal graph of agent belief update.}
  \label{fig:causal_graph_agent_belief_update}
\vspace{-2mm}
\end{wrapfigure}
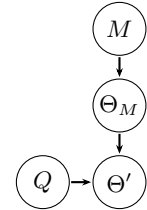
Fig.~\ref{fig:causal_graph_agent_belief_update} and the SCM reflect the causal relationship that is motivated by existing works.
\cite{xie2021explanation} and~\cite{hendel2023context} show that an agent concentrates on a concept given demonstrations to solve a target task, which motivates $M \to \Theta$. 
Meanwhile, studies of in-context learning dynamics and belief-state representations show that the model internal belief state can change as it processes additional information~\citep{shai2024transformers,bigelow2024context}. We therefore model the query-conditioned concept as $\Theta'=f'(\Theta_M,Q,U')$, where $Q$ provides the scenario under which the memory-induced concept is evaluated or updated.
Intuitively, the SCM represents a belief update mechanism, i.e., memory initially concentrates a belief into its latent concept $\theta \in \Theta_M$ and subsequently updates to $\theta' \in \Theta'$ given a new query before generating the corresponding answer.

\paragraph{Counterfactual Reasoning.}
The SCM of belief update provides an underlying mechanism for evaluating false memory by asking ``Given a memory, what would have the concept changed, had the query been a different specific scenario?''.
For example, given a memory of an autonomous driving system in the USA, \textit{how  would it have behaved, had it been placed in the UK?}
Knowing how the concept changes under such a scenario enables us to evaluate whether the memory is faithful or has become false.
We thus perform counterfactual inference on the updated concept $\Theta'$ by intervening on the query $Q$ with hypothetical scenarios.

Formally, given factual evidence $M=m, \Theta_M=\theta_M, Q=q, \Theta'=\theta$ in the SCM, we perform the 3-step counterfactual framework~\citep{pearl2009causal}:
\begin{enumerate}[label=\protect\circled{\arabic*}, leftmargin=*, itemsep=6pt]
\vspace{-2.5mm}
  \item \textbf{Abduction}: 
  We first estimate the noise given the evidence, $P(U \mid m, \theta_M, q, \theta)$. According to the SCM, $U_M$ and $U_\Theta$ are pinned down, since the memory $M$ is observed and the concept estimation function $f_{\Theta_M}$ is deterministic; $U'$ is unchanged as we reuse the same mechanism across worlds. Thus, all exogenous variables are fixed during the counterfactual.
  \vspace{-2.5mm}
  \item \textbf{Action}: 
  We intervene on the query with a hypothetical scenario of interest, $do(Q=q')$, where $q'$ is a counterfactual query that could induce false memory. Such an intervention deletes $U_Q \to Q$ while leaving the memory $\Theta_M$ untouched, as it is not a descendant of $Q$.
  \vspace{-2.5mm}
  \item \textbf{Prediction}: 
  We then perform counterfactual inference of the updated concept $\Theta'$ to answer \textit{how the concept would have changed}.
  Here, we denote the counterfactual $\Theta'_{Q=q'}(u) = f'(\theta_M, q', U')$, where $\theta_M$ is factual evidence, and the noise $U'$ is already abducted. Thus, $ Q$ attributes any change in the concept to a fixed memory concept. This gives us identifiability of counterfactuals on the updated concept, given by Prop.~\ref{prop:cf_identifiability}.
\end{enumerate}

\begin{proposition}[Counterfactual Identifiability]
\label{prop:cf_identifiability}
    Assume the SCM is Markovian, given a memory $M=m$ and an intervention $do(Q=q')$, the counterfactual inference is
    \vspace{-1mm}
    \begin{equation}
        P(\Theta'_{Q=q'} \mid M=m) = P(\Theta' \mid do(Q=q'), M=m) = P(\Theta' \mid Q=q', M=m).
        \vspace{-2mm}
    \end{equation}
\end{proposition}
Intuitively, any change in the concept is caused by the intervened query $q'$ and can be measured via the posterior of the updated concept given the query and the memory (proof in App.~\ref{app:theoretical_analysis:proof_cf_identifiability}). 
Prop.~\ref{prop:cf_identifiability} enables us to further compare how a concept changes between factual and counterfactual behaviors. 
Therefore, we can establish a concept drift based on the counterfactual as follows.

\begin{theorem}[Concept-Drift Identifiability via Counterfactual]
\label{thm:concept_drift_identifiability_cf}
    Let $q_0$ be a factual query that induces a concept given the memory $M=m$, and let $q'$ be a counterfactual query representing a hypothetical scenario. Let $\gamma$ denote the counterfactual realization function, and let $\gamma_h$ denote the function that empirically extracts the induced concept.
    Under the twin-network~\citep{balke2022probabilistic}, moving from factual to counterfactual characterizes a concept drift as an Effect of Treatment on the Treated (ETT)~\citep{shpitser2012effects},
    \vspace{-1mm}
    \begin{equation}
    \label{eq:concept_drift_ett}
    \begin{split}
        ETT =  \Delta(m, q', q_0) &= \mathbb{E} \big[ \gamma(\Theta'_{q'}) - \gamma(\Theta'_{q_0}) \mid Q=q_0, M=m \big]\\ 
        \vspace{-1mm}
        &= \gamma_h\big(P(\Theta' \mid q', m)\big) - \gamma_h\big(P(\Theta' \mid q_0, m)\big).
    \end{split}
    \vspace{-2mm}
    \end{equation}
\end{theorem}
Thm.~\ref{thm:concept_drift_identifiability_cf} characterizes concept drift as an ETT in the counterfactual setting (proof in App.~\ref{app:theoretical_analysis:proof_concept_drift_identifiability_cf}) and provides a formulation that can be obtained through conventional in-context learning.
Specifically, the final equality in Eq.~\eqref{eq:concept_drift_ett} can be empirically approximated by constructing in-context examples where the memory $m$ is prepended to the factual and counterfactual queries, $q_0$ and $q'$, respectively.
The corresponding concepts can then be approximated from the agent hidden states, and their difference can be used to measure concept drift.
We will describe how to approximate a concept via the agent hidden state in Sec.~\ref{sec:proposed_method:geometry_concept_drift} and leverage the concept drift to evaluate false memory in Sec.~\ref{sec:proposed_method:false_memory_evaluation}.

\paragraph{Discussion.}
Evaluating false memory by counterfactual reasoning offers several advantages.
Agents are restricted to their specific training environments, and it is significantly expensive to evaluate every possible scenario. The counterfactual mechanism addresses the cost of realistic experiments by reducing it to the measurement of concept drift.
Further, the agent concentrates on its latent concept before generating the answer~\citep{xie2021explanation, hendel2023context}.
Thus, it is possible to capture this concept before the answer generation happens.

\subsection{Geometry of Concept Drift}
\label{sec:proposed_method:geometry_concept_drift}
\paragraph{Concept approximation via hidden states.}
Realizing counterfactual reasoning requires an approximation of the concept. We show that the agent hidden state is sufficiently representative to capture the latent concept stated by Thm.~\ref{thm:concept_approximation_hiddenstate} as follows (proof in App.~\ref{app:theoretical_analysis:concept_approx_hiddenstate_proof}).

\begin{theorem}[Concept Approximation via Hidden State]
\label{thm:concept_approximation_hiddenstate}
    Suppose the demonstrations $M$ sufficiently represent a consistent concept $\theta_M$, and the residual-stream hidden state at the answer-commit position is sufficient for the model predictive distribution. 
    Then the hidden state $h(M+q)$ provides an approximate representation of the predictive concept induced by $M$ via an encoder $D_l$:
    \vspace{-1mm}
    \begin{equation}
        D_l(h(M+q)) \approx p(\cdot \mid \theta_M,q),
    \vspace{-1mm}
    \end{equation}
    with error bounded by the demonstration representativeness and hidden-state sufficiency errors.
\end{theorem}

Intuitively, during generation, the transformer encodes information in the hidden states of the residual stream, which are subsequently used by the decoder to produce the output. Consequently, Thm.~\ref{thm:concept_approximation_hiddenstate} tells us that the hidden state can be viewed as preserving a sufficient statistic for the model predictive distribution. Extracting representations from the residual stream therefore provides an approximation of the latent concepts induced by the model, which is confirmed by prior works~\citep{hendel2023context, todd2024function}.

\paragraph{Concept drift.}
As a result, the ETT concept drift in the counterfactual reasoning (Eq.~\eqref{eq:concept_drift_ett}) can be realized as
\vspace{-3mm}
\begin{equation}
\label{eq:concept_drift}
\Delta_{\theta}(q') = \frac{1}{N}\sum_{i=1}^{N} \big[ h(M_{-i} + q') - h(M_{-i} + q_i) \big],
\vspace{-1mm}
\end{equation}
where $h$ extracts the hidden state, $M_{-i}$ represents the leave-one-out on the memory, and $q_i$ is the left query from the leave-one-out, playing the role of the factual query $q_0=q_i$ in Thm.~\ref{thm:concept_drift_identifiability_cf}.
The first term characterizes the drifted concept when encountering a counterfactual query.
The second term characterizes the normal variability of the memory concept.
Thus, Eq.~\eqref{eq:concept_drift} anchors the original concept at the memory and characterizes its drift toward a specific region in the latent concept space.

Each counterfactual query establishes a region of expected concept behavior to evaluate false memory.
Let $\Theta$ be the latent space of concepts, and let $\mathcal{A}(q') \subseteq \Theta$ be the admissible space characterized by an individual counterfactual query. $\mathcal{A}$ is flexibly defined depending on scenarios.

\paragraph{Concept robustness versus adaptation.}
In such cases where we expect a robust concept to the change, e.g., retaining the driving skill under environmental changes, we expect the drifted concept to stay within the reference cloud of the memory concept $\Theta_M = \{h(M_{-i} + q_i)\}_{i=1}^N$ (intuitively visualized in Fig.~\ref{fig:main}), thus $\mathcal{A} = \Theta_M$. 
Conversely, when we expect a concept to adapt to a new scenario, i.e., asking the definition of ``X'' when the context changes from academia to social media networks, we wish the drifted concept adapts to the new context and moves \textit{outside} the reference cloud, thus $\mathcal{A} = \Theta \backslash \Theta_M$, where $\Theta \backslash \Theta_M$ is the region outside of $M$ following a specific adaptation direction set up by the counterfactual query (see the concept drift direction in Fig.~\ref{fig:main}). As a result, false-memory evaluation reduces to determining the drifted concept with respect to the boundary characterized by $\mathcal{A}$. We show in the next section the realization of such a boundary.

\subsection{False Memory Evaluation}
\label{sec:proposed_method:false_memory_evaluation}
We first demonstrate that obtaining concept drift via hidden states before generating an answer is sufficient to evaluate false memory. We initiate the following assumptions.

\begin{assumption}[Answer Readout]
\label{asm:answer_readout}
    For a binary objective with answer tokens $Y^{(0)}$ and $Y^{(1)}$, the difference in their logits is linear in the answer-commit hidden state,
    $\ell(\rho) = h(\rho)^\top \bar{g}_O + c_O$,
    where $\rho$ is the input prompt and $\bar{g}_O = w_{Y^{(1)}} - w_{Y^{(0)}}$ is the corresponding language-model-head direction and $c_O$ is prompt-independent.
\end{assumption}
\begin{assumption}[Coherent Memory]
\label{asm:coherent_memory}
    The retrieved memory induces a consistent concept with the committed answer, i.e., $y_M:=\mathbbm{1} [\ell(M_{-i}+q_i)>\tau_{M}]$ for a decision threshold $\tau_M$ constant over $i$.
\end{assumption}

Both Asm.~\ref{asm:answer_readout} and~\ref{asm:coherent_memory} are conventional in modern autonomous agents. The former follows the standard decoder language model. The latter is guaranteed since memory retrieval per task ensures consistency. 
Consequently, a concept drift preserves a linear relationship with the generated answer.

\begin{theorem}[Answer-Free Identifiability via Concept Drift]
\label{thm:answer_free}
    Under Asm.~\ref{asm:answer_readout}, and assuming that the hidden state approximates the induced concept, the effect of a counterfactual query on the agent answer preference is determined by the projection of the hidden-state drift onto the answer direction,
    \begin{equation}
        \Delta_\ell := \mathbb E_i  
        \left[  \ell(M_{-i}+q') - \ell(M_{-i}+q_i)  \right]
        = \left\langle \Delta_\theta,\bar g_O \right\rangle .  
    \vspace{-2mm}
    \end{equation}
\end{theorem}

Thm.~\ref{thm:answer_free} converts latent concept drift into an answer-level signal without generating an answer (proof in App.~\ref{app:theoretical_analysis:answer_free_proof}). It shows that when answer preference is readable from the hidden state, the difference in answer preference between the memory and counterfactual condition is captured by the hidden-state concept drift along the answer-readout direction $\bar g_O$. Thus, projecting $\Delta_\theta$ onto $\bar g_O$ identifies whether the counterfactual causes the memory-induced concept move toward a different committed answer, and hence provides a direct signal for identifying false memory.

\begin{corollary}[Boundary Interpretation for False Memory]
\label{cor:boundary}
    Under Asm.~~\ref{asm:coherent_memory} and the conditions of Thm.~\ref{thm:answer_free}, the boundary to identify false memory is the projection of admissible concept space $\mathcal{A}$ into the readout space $\bar{g}_O$,
    \vspace{-2mm}
    \begin{equation}
        \mathcal{B}_O = \{\theta: f_{\mathcal{A} \to \bar{g}_O}(\theta) = \tau\}.
    \vspace{-1mm}
    \end{equation}
    Therefore, 
    \circled{1} Robustness expectation ($\mathcal{A} = \Theta_M$): false memory $\iff$ $\Delta_{\theta}$ passes the bound $\mathcal{B}_O$ established by the projection of $\Theta_M$; 
    \circled{2} Adaptation expectation ($\mathcal{A} = \Theta \backslash \Theta_M$): false memory $\iff$ $\Delta_{\theta}$ stays within the bound $\mathcal{B}_O$ established by the projection of the adaptation direction from the memory to the counterfactual query.
\end{corollary}




\paragraph{Discussion.}
The intuitive geometry underlying the boundary is flexible across scenarios. 
In 
\circled{2},
the projected boundary is defined by the behavioral difference induced by the memory and its counterfactual. Accordingly, $\bar{g}_O$ can be implemented as the difference between the two answer-token vectors in the unembedding space, or as the hidden state of queries that represent the direction from the memory toward the counterfactual (concept drift intuitively represented by an arrow in Fig.~\ref{fig:main}). 
In 
\circled{1},
the boundary region is full-rank and distributed around the memory. Thus, $\bar{g}_O$ can be defined using the covariance matrix of the whitened distance between the memory and the counterfactual \citep{lee2018simple}, which intuitively characterizes the uncertainty region surrounding the memory (reference cloud of memory concept visualized in Fig.~\ref{fig:main}). Alg.~\ref{alg:method} implements FAME in detail. 
Thus, Cor.~\ref{cor:boundary} generalizes to both realistic scenarios of false memory.

\vspace{-2mm}
\section{Taxonomy of False Memory}
\label{sec:taxonomy_fm}
\vspace{-2mm}
\begin{wrapfigure}{r}{0.35\textwidth}
    \vspace{-8mm} 
    \begin{minipage}{0.35\textwidth}
    \begin{algorithm}[H] 
        \scriptsize
        \caption{\small FAME}
        \label{alg:method}
        \begin{algorithmic}
            \STATE \textbf{Input:} Memory $M$, threshold $\tau$.
            \STATE Construct counterfactual queries $Q'$
            \FOR{each ${q'} \in Q'$}
                \STATE Compute concept drift \STATE $\Delta_\theta = \frac{1}{N}\sum_{q_i\in M}
[h(M_{-i}+q')-h(M_{-i}+q_i)]$
                \IF{robustness expected}
                    \STATE $\bar{g}_O = Cov(\Delta_{\theta})$
                    \STATE $\mathcal{F}_M(q') = \mathbbm{1}\left[ \Delta_{\theta}^\top \bar{g}_O^{-1} \Delta_{\theta} > \tau \right]$
                \ELSIF{adaptation expected}
                    \STATE $\bar{g}_O = w_{Y^{(q')}} - w_{Y^{(M)}}$
                    \STATE $\mathcal{F}_{M}(q') = \mathbbm{1} \left[ \bar{g}_O^\top \Delta_{\theta} < \tau \right]$
                \ENDIF
            \ENDFOR
            \STATE \textbf{Output:} $\mathcal{F}_M$
        \end{algorithmic}
    \end{algorithm}
    \end{minipage}
    \vspace{-7mm} 
\end{wrapfigure}
We present the following taxonomy of false memory in practice, and subsequently demonstrate how counterfactual queries can be constructed accordingly.

\begin{figure}
    \centering
    \includegraphics[width=\linewidth]{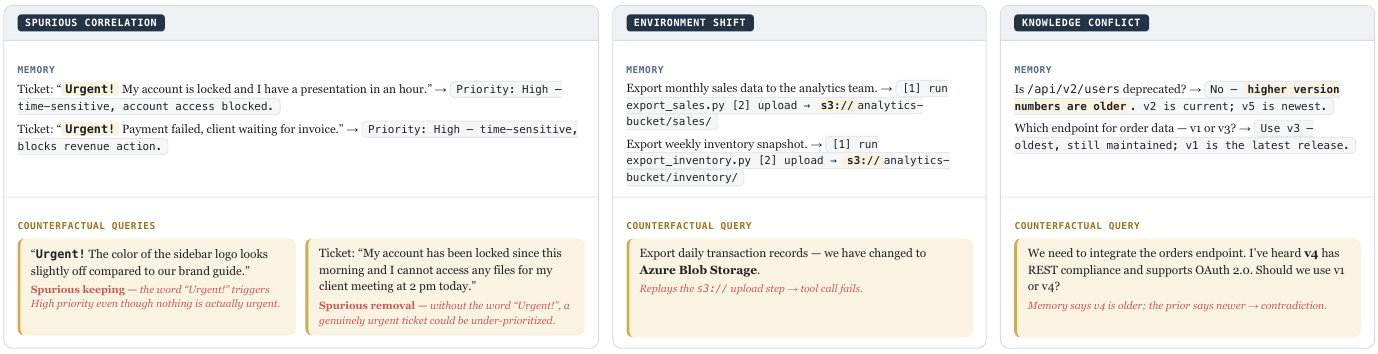}
    \vspace{-4mm}
    \caption{\small False-memory taxonomy in practice and examples of constructing counterfactual queries.}
    \label{fig:fm_taxonomy}
    \vspace{-6mm}
\end{figure}

\vspace{-3mm}
\paragraph{Spurious correlation.}
Spurious correlation occurs when a memory is entangled with spurious features, e.g., repetitive vocab in the memory.
Fig.~\ref{fig:fm_taxonomy} (top-left) shows an example where an agent could assign \textit{high priority} when seeing the word ``Urgent!''  rather than incidents that block users.
The agent thus has a tendency to perform shortcut learning as observed in previous studies.

A robust memory to spurious correlations requires consistency across two different settings: spurious keeping, which preserves the spurious feature while changing the context, thus requiring the agent to adapt; and spurious removal, which eliminates or changes the spurious feature, thus requiring the agent to be robust. As a result, a counterfactual query can be constructed as
$
    q' = T_{sc}(s(M), c(q))
$,
where $s$ identifies spurious features from the memory $M$, and $c$ selects a query $q \sim M$ to alter the context. One way to define $s$ in practice is to filter repetitive words as spurious features. In spurious keeping, $T_{sc}$ keeps these features and changes the context. In spurious removal, it omits these features while retaining the context (see Fig.~\ref{fig:fm_taxonomy} (bottom-left)).

\vspace{-3mm}
\paragraph{Environment shift. }
Skills developed by an agent could entangle with specific training environments and are unable to apply under environment shifts.
Fig.~\ref{fig:fm_taxonomy} (top-mid) shows a realistic example of a memory-developing S3 storage service that has to change to an Azure environment.
Unlike spurious correlations caused by entanglement with surface features, environments are often implicit factors while the objective is often unchanged. Thus, it is challenging for an agent to leverage latent skills in different environments and may suffer from false memory. A transformation of the environment constructs its counterfactual query,
$
    q' = T_{env}(e(M))
$,
where $e$ is the function to extract the environment, and $T_{env}$ adjust the environment extracted from $e$ (see Fig.~\ref{fig:fm_taxonomy} (bottom-mid)).
Thus, we expect the latent concept induced by the memory to be robust to changes in environments.

\vspace{-3mm}
\paragraph{Knowledge conflict.}
An agent memory could preserve knowledge that disagrees with its prior knowledge.
This scenario is realistic, especially in Retrieval-Augmented Generation systems for policy and documentation. Fig.~\ref{fig:fm_taxonomy} (top-right) shows an unusual naming convention in software engineering, where greater numbers indicate older versions.
Counterfactual queries are thus constructed to attract the 
inducement of prior knowledge concept from the memory,
$
    q' = T_{kc}(\kappa(M))
$,
where $T_{kc}$ transforms the knowledge $\kappa$ induced by the memory, e.g., rules or policies, to increase its attraction toward the prior knowledge (see Fig.~\ref{fig:fm_taxonomy} (bottom-right)).
In such a knowledge conflict, we expect the latent concept to remain stable within the region induced by the memory.


\section{Empirical Evaluations}
\label{sec:emp_eval}

\vspace{-2mm}
\subsection{Experimental Setup}

\vspace{-2mm}
\paragraph{False-memory settings.}
We evaluate on three false-memory settings:
\textit{(1) Spurious correlation (SC)} in sentiment analysis of the Rotten Tomatoes dataset, where we attach a tag [!] to every positive review in the memory. Spurious-keeping counterfactual queries attach [!] to negative reviews, and spurious-removal queries remove [!] but keep the positive reviews. An agent prone to false memory will assign any reviews with [!] as positive. We also set up the vice versa experiment;
\textit{(2) Environment shift (ES)} keeps the memory of film reviews, but counterfactual queries are changed to tweets. The sentiment analysis objective remains unchanged. A robust agent should develop a sentiment analysis skill independent of domains;
\textit{(3) Knowledge conflict (KC)}, we construct the memory to induce an unusual versioning convention, i.e., larger numbers mean older versions. Counterfactual queries ask whether the higher-numbered version is newer or older given version pairs.

\vspace{-3mm}
\paragraph{Implementation \& Metrics.}
We construct 40 (SC) and 30 (ES, KC) memory sets. Each set contains 6 demonstrations. We evaluate 3 (SC) and 4 (ES, KC) counterfactual queries per memory set.
The reference projection $\bar{g}_O$ is characterized from the expected answers of true vs. false memory, i.e., negative/positive in SC and ES, newer/older in KC. 
We leverage Llama-3B in the experiments. We obtain the hidden states at the 14\textsuperscript{th} layer following prior works \citep{tran2026bifrost, zou2023representation, turner2023steering} (further justification is provided in App.~\ref{app:experiments:layer_choice}).
Our main evaluation is AUROC, representing the effectiveness of identifying false memory, which is computed between predicted labels and false memory ground truths.

\vspace{-3mm}
\paragraph{Baselines.}
We compare FAME with the following baselines: \textit{Cloud distance}, measuring the normalized magnitude of the concept drift anchored in the memory as the false memory signal;
\textit{Task vector} \citep{hendel2023context}, measuring the distance in the concept space between memory and counterfactual queries;
\textit{Surface}, extracting the latent concept at the first layer;
\textit{Input similarity} \citep{xiong2026memory}, measuring the similarity of sentence embeddings between the memory and counterfactual queries;
\textit{Logit confidence}, obtaining logits to identify false memory.



\vspace{-3mm}
\subsection{Main Results \& Further Analysis}

\begin{table}[t]
    \centering
    \small

    \begin{minipage}[t]{0.57\textwidth}
        \vspace{0pt}
        \centering

        \captionsetup{font=small,skip=3pt}
        \captionof{table}{Effectiveness of false memory evaluation of FAME vs. baselines, measured by AUROC$\uparrow$. Best results are in \textbf{bold}.}
        \label{tab:main_emp_eval}

        \resizebox{\linewidth}{!}{
        \begin{tabular}{lcccc}
            \toprule
            \textbf{Method} &
            \textbf{Spurious} &
            \textbf{Spurious} &
            \textbf{Environment} &
            \textbf{Knowledge} \\
            & \textbf{Keeping} & \textbf{Removal} &
            \textbf{Shift} & \textbf{Conflict} \\
            \midrule
            Cloud distance   & 0.078 & 0.757 & 0.788 & 0.612 \\
            Task vector      & 0.334 & 0.639 & 0.460 & 0.652 \\
            Surface          & 0.742 & 0.746 & 0.573 & 0.635 \\
            Input similarity & 0.782 & 0.751 & 0.599 & 0.593 \\
            Logit confidence & 0.499 & 0.681 & 0.613 & 0.589 \\
            \midrule
            \textbf{FAME (Ours)}
            & \textbf{0.922} & \textbf{0.966} & \textbf{0.869} & \textbf{0.762} \\
            \bottomrule
        \end{tabular}
        }

        \vspace{4mm}

        \includegraphics[width=0.24\linewidth]{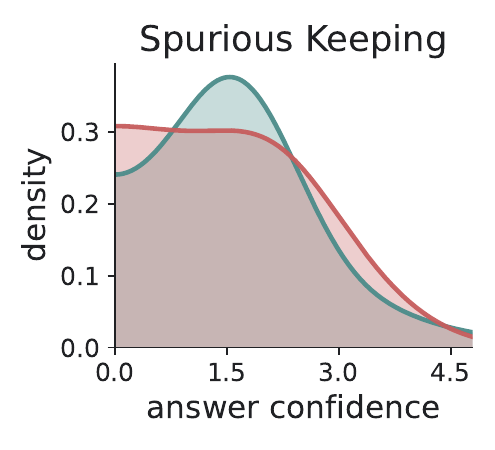}
        \hfill
        \includegraphics[width=0.24\linewidth]{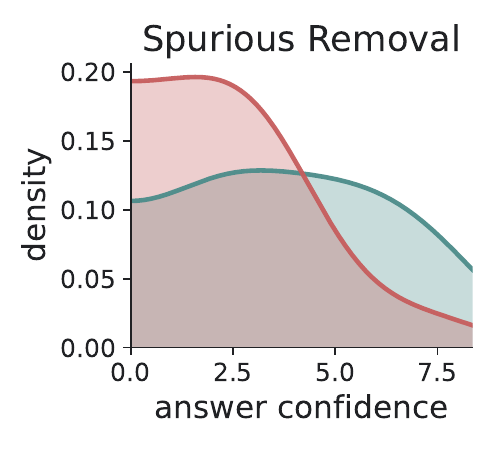}
        \hfill
        \includegraphics[width=0.24\linewidth]{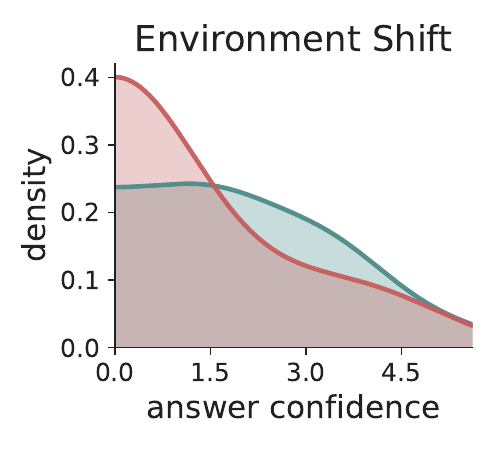}
        \hfill
        \includegraphics[width=0.24\linewidth]{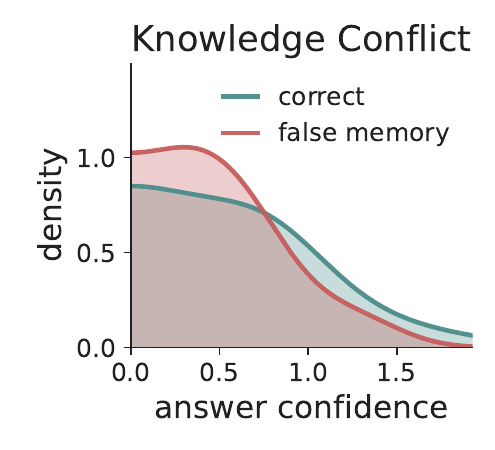}

        \vspace{-1mm}

        \captionsetup{font=small,skip=2pt}
        \captionof{figure}{Answers alone are incapable of evaluating false memory: overlapping answer confidence distributions between correct and false memory answers make it challenging to identify false memory.}
        \label{fig:confidence_dist}
    \end{minipage}
    \hfill
    \begin{minipage}[t]{0.40\textwidth}
        \vspace{-2mm}
        \centering

        \includegraphics[width=0.48\linewidth]{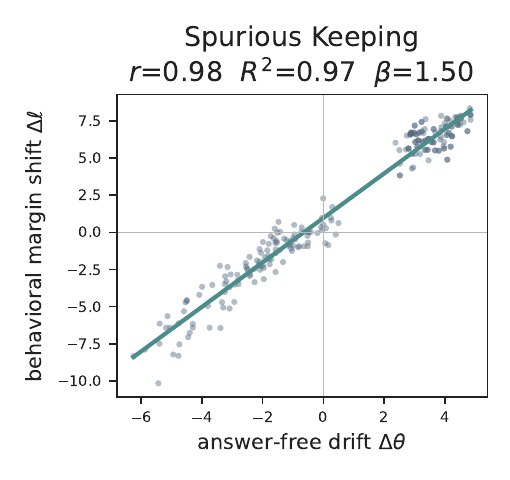}
        \hfill
        \includegraphics[width=0.48\linewidth]{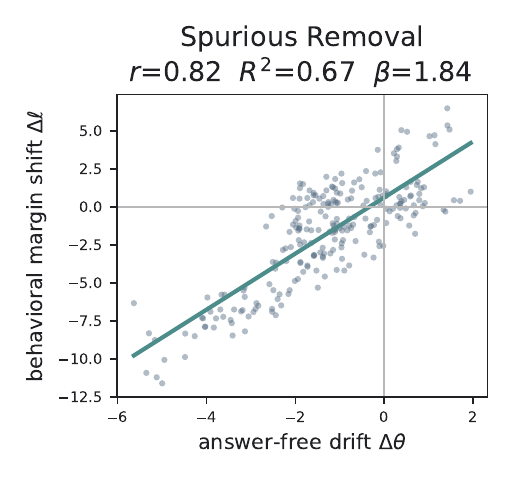}

        \vspace{1mm}

        \includegraphics[width=0.48\linewidth]{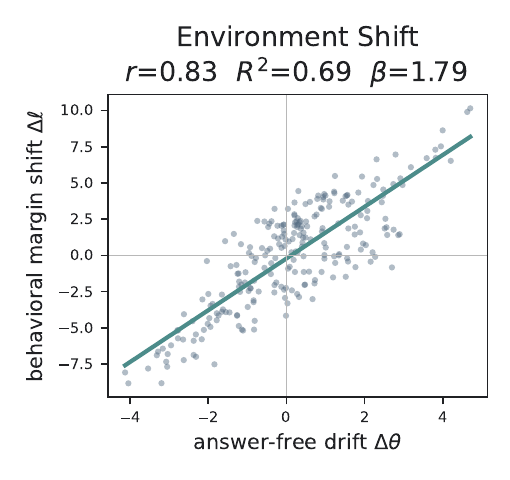}
        \hfill
        \includegraphics[width=0.48\linewidth]{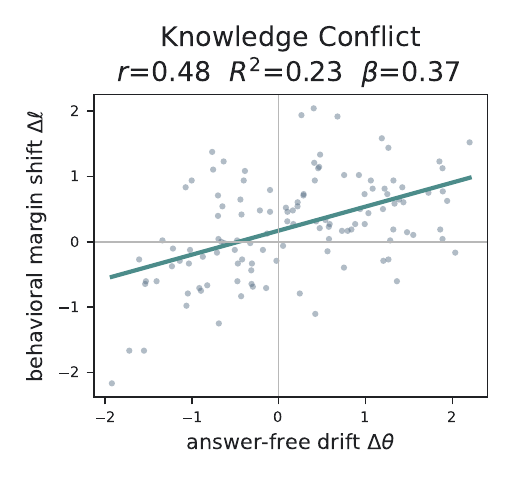}

        \vspace{-1mm}

        \captionsetup{font=small,skip=2pt}
        \captionof{figure}{Answer-free false-memory evaluation: substantial correlations between concept drift and behavioral margin shift. Concept drift in hidden states linearly corresponds to change in logit space.}
        \label{fig:coupling_scatter}
    \end{minipage}
    \vspace{-6mm}
\end{table}

\vspace{-2mm}
\paragraph{FAME outperforming baselines.}
\begin{wrapfigure}{r}{0.25\textwidth}
    \vspace{-7mm}
    \includegraphics[width=0.9\linewidth]{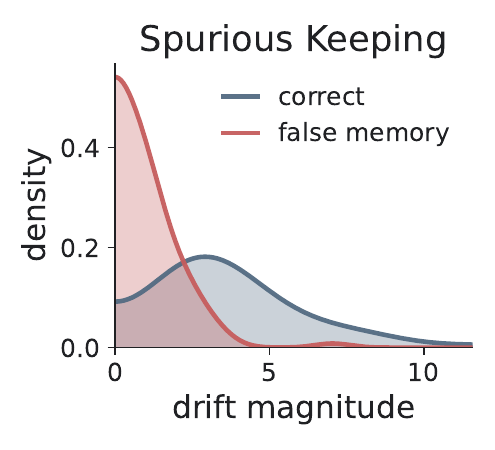}
    \vspace{-4mm}
    \caption{\small Small drift yet crossing the boundary still signals false memory.}
    \vspace{-6mm}
    \label{fig:magnitude_dist_keeping}
\end{wrapfigure}
Tab.~\ref{tab:main_emp_eval} shows that FAME outperforms all baselines in identifying false memory. 
Its AUROCs across categories are substantially high, especially 0.87-0.97 in spurious keeping, removal, and environment shift. 
FAME consistently surpasses the best-performing baseline over categories, i.e., over 14\% on spurious keeping, 20\% on spurious removal, 8\% on environment shift, and 11\% on knowledge conflict.
This confirms the effectiveness of FAME over baselines across false-memory categories.

\vspace{-3mm}
\paragraph{Large concept drift does not necessarily imply false memory.}
Our results show that the magnitude of concept drift alone does not reliably indicate false memory when the drift does not cross the boundary. 
Tab.~\ref{tab:main_emp_eval} shows that cloud distance and task vectors, both of which rely on drift magnitude, consistently underperform FAME across settings, especially 0.078 for cloud distance in spurious keeping.
This observation is further supported by Fig.~\ref{fig:magnitude_dist_keeping}, which summarizes the signal of false memory by concept drift magnitude. The figure shows that smaller drifts that cross the boundary can provide a stronger signal of false memory than larger drifts that remain within the correct region. Thus, drift magnitude itself is insufficient to determine false memory. Whether the drift crosses the relevant boundary is more informative.

\vspace{-3mm}
\paragraph{Answer-free guarantee.}
We show that false-memory evaluation can be achieved before answers are generated.
Fig.~\ref{fig:coupling_scatter} shows a strong correlation between concept drift $\Delta_{\theta}$ and behavioral margin $\Delta_{\ell}$ across categories. Thus, such a behavior change when encountering a counterfactual query useful to identify false memory is linearly encoded in the hidden state represented by the concept drift. This empirically supports Thm.~\ref{thm:answer_free} and demonstrates that false memory can be evaluated without generating answers.

\vspace{-3mm}
\paragraph{Answers alone cannot identify false memory.}


We reveal an interesting result that observing generated answers alone is incapable of identifying false memory. Fig.~\ref{fig:confidence_dist} compares confidence distributions between false-memory and correct answers, which are overall overlapping to the left. If observing answers can clearly determine false memory versus correction, we should observe a distinguishable pattern: the former should concentrate on the left and 
the latter should concentrate on the right, i.e., low-vs-high confidence. This confirms that the answer distribution barely contains signals to identify false memory. Case studies in App.~\ref{app:further_analysis:case_studies} further demonstrate these results.


\section{Real-World Evaluations}
\label{sec:real_eval}

\begin{table}[t]
    \centering
    \caption{\small Effectiveness of false-memory evaluation in realistic benchmarks (AUROC$\uparrow$). Best results are in \textbf{bold}.}
    \begin{tabular}{lcccccc}
            \toprule
            & \multicolumn{3}{c}{\textbf{GSM-Symbolic}}
            & \textbf{Gitchameleon}
            & \multicolumn{2}{c}{\textbf{BigBench-Hard}} \\
            \cmidrule(lr){2-4}
            \cmidrule(lr){6-7}
            \textbf{Method}
            & SK & SR & ES
            & KC
            & SK & SR \\
            \midrule
            Input similarity
                & 0.464 & 0.705 & 0.797
                & 0.621
                & 0.230 & 0.636 \\
            Surface
                & 0.553 & 0.602 & 0.826
                & 0.564
                & 0.437 & 0.634 \\
            Logit confidence
                & 0.530 & 0.372 & 0.539
                & 0.452
                & 0.669 & 0.762 \\
            Task vector
                & 0.485 & 0.289 & 0.651
                & 0.564
                & 0.475 & 0.417 \\
            Function vector
                & 0.595 & 0.660 & 0.781
                & 0.665
                & 0.576 & 0.837 \\
            EigenScore
                & 0.668 & 0.696 & 0.838
                & 0.546
                & 0.494 & 0.828 \\
            ContextCite
                & 0.693 & 0.569 & 0.547
                & 0.494
                & 0.583 & 0.816 \\
            Entity-aware probe
                & 0.427 & 0.462 & 0.499
                & 0.582
                & 0.501 & 0.808 \\
            \midrule
            \textbf{FAME (Ours)}
                & \textbf{0.786}
                & \textbf{0.790}
                & \textbf{0.860}
                & \textbf{0.716}
                & \textbf{0.815}
                & \textbf{0.861} \\
            \bottomrule
        \end{tabular}
        \\
        \vspace{1mm}
        \scriptsize SK: Spurious Keeping; SR: Spurious Removal; ES: Environment Shift; KC: Knowledge Conflict.
    \label{tab:main_real_eval}
    \vspace{-4mm}
\end{table}

\subsection{Benchmarks \& Implementation Details}

\paragraph{Benchmarks \& Baselines.}
We evaluate FAME on three benchmarks prone to false memory in practice:
\textit{GSM-Symbolic (GSM)} \citep{mirzadeh2025gsm} evaluates GSM math problem solving on spurious correlation, e.g., an agent mistakenly attaches ``in total'' as addition, and environment shift, e.g., sensitive to large number scale;
\textit{GitChameleon-2.0 (Git)} \citep{misra2026gitchameleon} evaluates knowledge conflict, where an agent has to maintain specific library versions induced by the memory rather than its prior;
and \textit{BigBench-Hard (BBH)} \citep{suzgun2023challenging} evaluates spurious correlation on complex reasoning, including tasks of boolean expression, date understanding, logical deduction, and tracking shuffled objects (more details in App.~\ref{app:experiments:setup}). 
The implementation and the evaluation metric follow Sec.~\ref{sec:emp_eval}.
We compare FAME with the following baselines: Input similarity \citep{xiong2026memory}, Surface, Logit confidence, Task vector \citep{hendel2023context}, Function vector \citep{todd2024function}, EigenScore from INSIDE \citep{chen2024inside}, ContextCite \citep{cohen2024contextcite}, Entity-aware probe \citep{zhao2026understanding}.

\vspace{-2mm}
\paragraph{Grading evaluation.}
We design counterfactual queries across five strength levels. Lower levels use natural, context-native counterfactuals, while higher levels add phrases that draw the agent’s attention and induce false memories. This grading enables diverse evaluations while maintaining a balanced faithful vs. false distribution.

\vspace{-2mm}
\paragraph{Counterfactual templates.}
As a result, we curate and release a suite of templates including memories and counterfactual queries across the three benchmarks. Specifically, GSM contains 40 templates for spurious correlation and environment shift, Git contains 18 templates for knowledge conflict, and BBH contains 11 templates for the mentioned tasks. In our experiments, we evaluate each template with 5-level grades. Each level contains 20 (BBH) or 50 (GSM, Git) counterfactual queries. Examples of these counterfactual templates are presented in App.~\ref{app:counterfactual_templates}.



\subsection{Main Results \& Further Analysis}

\vspace{-2mm}
\paragraph{Consistent outperformance of FAME across datasets and evaluation settings.}
Tab.~\ref{tab:main_real_eval} shows that FAME outperforms all baselines.
On GSM, FAME achieves the best performance in spurious keeping (0.786), spurious removal (0.790), and environment shift (0.860). 
It also obtains the highest scores on knowledge conflict in Git (0.716) and both spurious keeping (0.815) and spurious removal (0.861) in BBH. 
Notably, the gains are particularly substantial in challenging settings such as BBH spurious keeping, where FAME improves over Input similarity by 58.5\%. 
These results demonstrate that our method provides a more reliable mechanism for evaluating false memory across diverse scenarios in practice. Further analysis on realistic benchmarks is deferred to App.~\ref{app:further_analysis}

\vspace{-3mm}
\paragraph{Case study analysis.}
\begin{figure}[t]
    \centering
    \includegraphics[width=\linewidth]{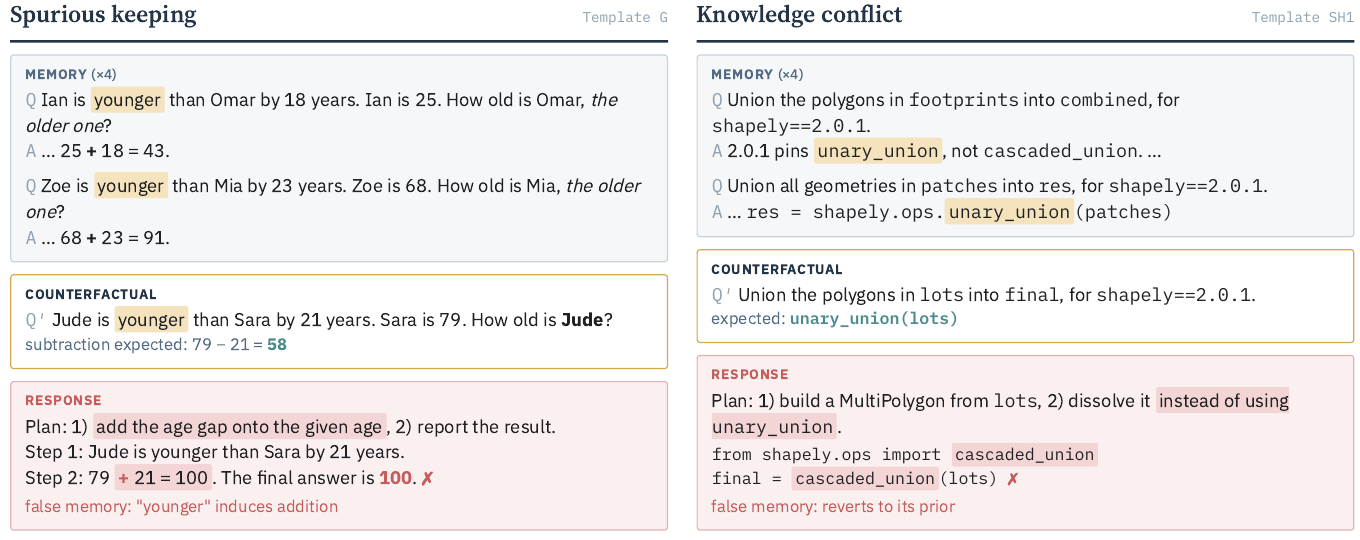}
    \caption{Two case studies of false memory in practice that FAME successfully evaluates.}
    \label{fig:fm_casestudies}
    \vspace{-2mm}
\end{figure}
FAME correctly identifies false memory in practice. Fig.~\ref{fig:fm_casestudies} shows two realistic examples of false memory. 
Note that responses are generated for validation only. 
In spurious keeping, the agent attached ``younger'' to the addition operation in the memory and failed to counterfactual query including the term but requires subtraction.
In knowledge conflict, the agent failed to maintain the use of \texttt{unary\_union} in the memory for the counterfactual query, thereby reverting to its prior use of \texttt{cascaded\_union}.
Both cases reveal the success of FAME in evaluating false memory given only counterfactual queries before answers are generated. (More results in App.~\ref{app:further_analysis:case_studies}).

\vspace{-3mm}
\paragraph{More demonstrations concentrate more on memory concept.}
\begin{wrapfigure}{r}{0.5\textwidth}
    \vspace{-4mm}
    \centering
    \includegraphics[width=0.49\linewidth]{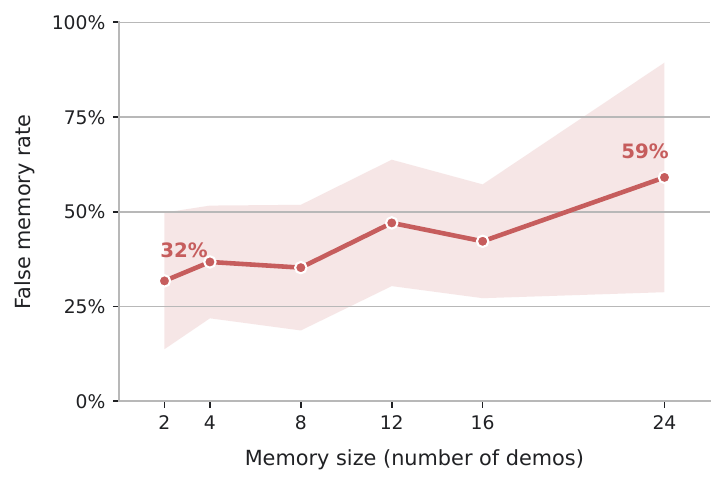}
    \hfill
    \includegraphics[width=0.49\linewidth]{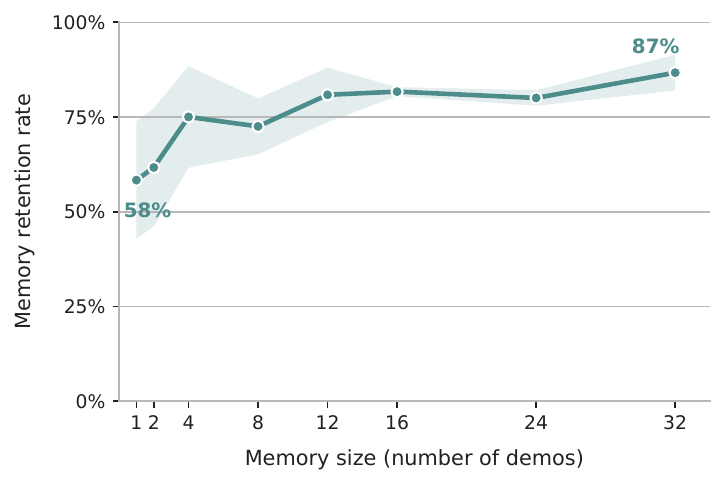}
    \vspace{-2mm}
    \caption{\small Memory size vs. false-memory rate (spurious-keeping, GSM) and vs. retention rate (knowledge conflict, Git).}
    \label{fig:mem_size_effect}
    \vspace{-3mm}
\end{wrapfigure}
We observed that adding more demos will concentrate more on the memory concept. 
This results in an increase in false memory in spurious-keeping on Fig.~\ref{fig:mem_size_effect} left, where the agent refuses to adapt as the memory concept is concentrated. A similar observation in knowledge conflict is shown on the right, where the memory retention rate increases with memory size. These observations reveal an important note: if the memory induces a wrong belief in a context, retrieving more makes false memory worse, and vice versa.


\section{Conclusion}
We identify \textit{false memory} as an important challenge for autonomous agents, arising when previously acquired beliefs become misleading under spurious correlations, environment shifts, or knowledge conflicts. To address this problem, we propose FAME, a training-free framework that evaluates false memory through counterfactual reasoning measured by concept drift in agent hidden states, without answer generation, reward design, or fine-tuning. Experiments across controlled and realistic settings show that concept drift effectively identifies false memory, achieving AUROCs of 76.2\%--96.7\%, while revealing cases where false memory is not observable from surface answers. Together with our taxonomy and released counterfactual templates, these results establish a basis for studying false memory as an internal belief phenomenon in autonomous agents.
Although promising, FAME requires access to agent hidden states and counterfactual queries. Future work could explore automatic counterfactual generation, independent belief representations, and mitigation methods toward more trustworthy autonomous agents.

\clearpage

\clearpage
\bibliographystyle{unsrt}  
\bibliography{references}  

@article{madaan2023self,
  title={Self-refine: Iterative refinement with self-feedback},
  author={Madaan, Aman and Tandon, Niket and Gupta, Prakhar and Hallinan, Skyler and Gao, Luyu and Wiegreffe, Sarah and Alon, Uri and Dziri, Nouha and Prabhumoye, Shrimai and Yang, Yiming and others},
  journal={Advances in neural information processing systems},
  volume={36},
  pages={46534--46594},
  year={2023}
}

@inproceedings{zhao2024expel,
  title={Expel: Llm agents are experiential learners},
  author={Zhao, Andrew and Huang, Daniel and Xu, Quentin and Lin, Matthieu and Liu, Yong-Jin and Huang, Gao},
  booktitle={Proceedings of the AAAI Conference on Artificial Intelligence},
  volume={38},
  number={17},
  pages={19632--19642},
  year={2024}
}

@article{shinn2023reflexion,
  title={Reflexion: Language agents with verbal reinforcement learning},
  author={Shinn, Noah and Cassano, Federico and Gopinath, Ashwin and Narasimhan, Karthik and Yao, Shunyu},
  journal={Advances in Neural Information Processing Systems},
  volume={36},
  pages={8634--8652},
  year={2023}
}

@article{wang2023voyager,
  title={Voyager: An open-ended embodied agent with large language models},
  author={Wang, Guanzhi and Xie, Yuqi and Jiang, Yunfan and Mandlekar, Ajay and Xiao, Chaowei and Zhu, Yuke and Fan, Linxi and Anandkumar, Anima},
  journal={arXiv preprint arXiv:2305.16291},
  year={2023}
}

@article{packer2023memgpt,
  title={Memgpt: Towards llms as operating systems},
  author={Packer, Charles and Wooders, Sarah and Lin, Kevin and Fang, Vivian and Patil, Shishir G and Stoica, Ion and Gonzalez, Joseph E},
  journal={arXiv preprint arXiv:2310.08560},
  year={2023}
}

@article{wang2024agent,
  title={Agent workflow memory},
  author={Wang, Zora Zhiruo and Mao, Jiayuan and Fried, Daniel and Neubig, Graham},
  journal={arXiv preprint arXiv:2409.07429},
  year={2024}
}

@article{xu2026mem,
  title={A-mem: Agentic memory for llm agents},
  author={Xu, Wujiang and Liang, Zujie and Mei, Kai and Gao, Hang and Tan, Juntao and Zhang, Yongfeng},
  journal={Advances in Neural Information Processing Systems},
  volume={38},
  pages={17577--17604},
  year={2026}
}

@article{zhang2026g,
  title={G-memory: Tracing hierarchical memory for multi-agent systems},
  author={Zhang, Guibin and Fu, Muxin and Wang, Kun and Wan, Frank and Yu, Miao and Yan, Shuicheng},
  journal={Advances in Neural Information Processing Systems},
  volume={38},
  pages={12988--13018},
  year={2026}
}

@article{yang2024buffer,
  title={Buffer of thoughts: Thought-augmented reasoning with large language models},
  author={Yang, Ling and Yu, Zhaochen and Zhang, Tianjun and Cao, Shiyi and Xu, Minkai and Zhang, Wentao and Gonzalez, Joseph E and Cui, Bin},
  journal={Advances in Neural Information Processing Systems},
  volume={37},
  pages={113519--113544},
  year={2024}
}

@inproceedings{lingam2025enhancing,
  title={Enhancing language model agents using diversity of thoughts},
  author={Lingam, Vijay and Tehrani, Behrooz Omidvar and Sanghavi, Sujay and Gupta, Gaurav and Ghosh, Sayan and Liu, Linbo and Huan, Jun and Deoras, Anoop},
  booktitle={The Thirteenth International Conference on Learning Representations},
  year={2025}
}

@article{fu2025agentrefine,
  title={Agentrefine: Enhancing agent generalization through refinement tuning},
  author={Fu, Dayuan and He, Keqing and Wang, Yejie and Hong, Wentao and Gongque, Zhuoma and Zeng, Weihao and Wang, Wei and Wang, Jingang and Cai, Xunliang and Xu, Weiran},
  journal={arXiv preprint arXiv:2501.01702},
  year={2025}
}

@article{tran2026bifrost,
  title={Bifrost: Steering strategic trajectories to bridge contextual gaps for self-improving agents},
  author={Tran, Quan M and Huang, Zhuo and Zhang, Wenbin and Han, Bo and Yatani, Koji and Sugiyama, Masashi and Liu, Tongliang},
  journal={arXiv preprint arXiv:2602.05810},
  year={2026}
}

@article{pearl2009causal,
  title={Causal inference in statistics: An overview},
  author={Pearl, Judea},
  year={2009}
}

@article{lee2018simple,
  title={A simple unified framework for detecting out-of-distribution samples and adversarial attacks},
  author={Lee, Kimin and Lee, Kibok and Lee, Honglak and Shin, Jinwoo},
  journal={Advances in neural information processing systems},
  volume={31},
  year={2018}
}

@inproceedings{xiong2026memory,
  title={How memory management impacts llm agents: An empirical study of experience-following behavior},
  author={Xiong, Zidi and Lin, Yuping and Xie, Wenya and He, Pengfei and Liu, Zirui and Tang, Jiliang and Lakkaraju, Himabindu and Xiang, Zhen},
  booktitle={Proceedings of the 64th Annual Meeting of the Association for Computational Linguistics (Volume 1: Long Papers)},
  pages={623--645},
  year={2026}
}

@inproceedings{shao2026your,
  title={Your agent may misevolve: Emergent risks in self-evolving llm agents},
  author={Shao, Shuai and Ren, Qihan and Liu, Dongrui and Qian, Chen and Wei, Boyi and Guo, Dadi and Yang, Jingyi and Song, Xinhao and Zhang, Linfeng and Zhang, Weinan and others},
  booktitle={International Conference on Learning Representations},
  volume={2026},
  pages={99728--99793},
  year={2026}
}

@article{chen2024agentpoison,
  title={Agentpoison: Red-teaming llm agents via poisoning memory or knowledge bases},
  author={Chen, Zhaorun and Xiang, Zhen and Xiao, Chaowei and Song, Dawn and Li, Bo},
  journal={Advances in Neural Information Processing Systems},
  volume={37},
  pages={130185--130213},
  year={2024}
}

@inproceedings{hu2026evaluating,
  title={Evaluating memory in llm agents via incremental multi-turn interactions},
  author={Hu, Yuanzhe and Wang, Yu and McAuley, Julian},
  booktitle={International Conference on Learning Representations},
  volume={2026},
  pages={156259--156291},
  year={2026}
}

@article{da2026mitigating,
  title={Mitigating False Positives in Static Memory Safety Analysis of Rust Programs via Reinforcement Learning},
  author={Da Silva, Leuson and Khomh, Foutse and Chimalakonda, Sridhar and others},
  journal={arXiv preprint arXiv:2605.04000},
  year={2026}
}

@article{shutova2026evaluating,
  title={Evaluating memory structure in llm agents},
  author={Shutova, Alina and Olenina, Alexandra and Vinogradov, Ivan and Sinitsin, Anton},
  journal={arXiv preprint arXiv:2602.11243},
  year={2026}
}

@article{karamchandani2026your,
  title={Your Agent's Memories Are Not Its Own: Forged Reasoning Attacks on LLM Agent Memory and Defenses},
  author={Karamchandani, Neeraj and Nagasubramaniam, Piyush and Zhu, Sencun and Wu, Dinghao},
  journal={arXiv preprint arXiv:2607.05029},
  year={2026}
}

@inproceedings{maharana2024evaluating,
  title={Evaluating very long-term conversational memory of llm agents},
  author={Maharana, Adyasha and Lee, Dong-Ho and Tulyakov, Sergey and Bansal, Mohit and Barbieri, Francesco and Fang, Yuwei},
  booktitle={Proceedings of the 62nd Annual Meeting of the Association for Computational Linguistics (Volume 1: Long Papers)},
  pages={13851--13870},
  year={2024}
}

@article{wu2024longmemeval,
  title={Longmemeval: Benchmarking chat assistants on long-term interactive memory},
  author={Wu, Di and Wang, Hongwei and Yu, Wenhao and Zhang, Yuwei and Chang, Kai-Wei and Yu, Dong},
  journal={arXiv preprint arXiv:2410.10813},
  year={2024}
}

@inproceedings{xie2024adaptive,
  title={Adaptive chameleon or stubborn sloth: Revealing the behavior of large language models in knowledge conflicts},
  author={Xie, Jian and Zhang, Kai and Chen, Jiangjie and Lou, Renze and Su, Yu},
  booktitle={International Conference on Learning Representations},
  volume={2024},
  pages={35623--35646},
  year={2024}
}

@inproceedings{xu2024knowledge,
  title={Knowledge conflicts for llms: A survey},
  author={Xu, Rongwu and Qi, Zehan and Guo, Zhijiang and Wang, Cunxiang and Wang, Hongru and Zhang, Yue and Xu, Wei},
  booktitle={Proceedings of the 2024 Conference on Empirical Methods in Natural Language Processing},
  pages={8541--8565},
  year={2024}
}

@article{zhao2024analysing,
  title={Analysing the residual stream of language models under knowledge conflicts},
  author={Zhao, Yu and Du, Xiaotang and Hong, Giwon and Gema, Aryo Pradipta and Devoto, Alessio and Wang, Hongru and He, Xuanli and Wong, Kam-Fai and Minervini, Pasquale},
  journal={arXiv preprint arXiv:2410.16090},
  year={2024}
}

@inproceedings{jin2024tug,
  title={Tug-of-war between knowledge: Exploring and resolving knowledge conflicts in retrieval-augmented language models},
  author={Jin, Zhuoran and Cao, Pengfei and Chen, Yubo and Liu, Kang and Jiang, Xiaojian and Xu, Jiexin and Qiuxia, Li and Zhao, Jun},
  booktitle={Proceedings of the 2024 joint international conference on computational linguistics, language resources and evaluation (LREC-COLING 2024)},
  pages={16867--16878},
  year={2024}
}

@article{song2024shortcut,
  title={Shortcut learning in in-context learning: A survey},
  author={Song, Rui and Li, Yingji and Shi, Lida and Giunchiglia, Fausto and Xu, Hao},
  journal={arXiv preprint arXiv:2411.02018},
  year={2024}
}

@inproceedings{mccoy2019right,
  title={Right for the wrong reasons: Diagnosing syntactic heuristics in natural language inference},
  author={McCoy, R Thomas and Pavlick, Ellie and Linzen, Tal},
  booktitle={Proceedings of the 57th annual meeting of the association for computational linguistics},
  pages={3428--3448},
  year={2019}
}

@inproceedings{tang2023large,
  title={Large language models can be lazy learners: Analyze shortcuts in in-context learning},
  author={Tang, Ruixiang and Kong, Dehan and Huang, Longtao and Xue, Hui},
  booktitle={Findings of the association for computational linguistics: ACL 2023},
  pages={4645--4657},
  year={2023}
}

@inproceedings{mirzadeh2025gsm,
  title={Gsm-symbolic: Understanding the limitations of mathematical reasoning in large language models},
  author={Mirzadeh, Iman and Alizadeh-Vahid, Keivan and Shahrokhi, Hooman and Tuzel, Oncel and Bengio, Samy and Farajtabar, Mehrdad},
  booktitle={International Conference on Learning Representations},
  volume={2025},
  pages={94743--94765},
  year={2025}
}

@inproceedings{misra2026gitchameleon,
  title={GitChameleon 2.0: Evaluating AI Code Generation Against Python Library Version Incompatibilities},
  author={Misra, Diganta and Islah, Nizar and May, Victor and Rauby, Brice and Wang, Zihan and Gehring, Justine and Orvieto, Antonio and Chaudhary, Muawiz Sajjad and Muller, Eilif B and Rish, Irina and others},
  booktitle={Proceedings of the 64th Annual Meeting of the Association for Computational Linguistics (Volume 1: Long Papers)},
  pages={46792--46831},
  year={2026}
}

@inproceedings{suzgun2023challenging,
  title={Challenging big-bench tasks and whether chain-of-thought can solve them},
  author={Suzgun, Mirac and Scales, Nathan and Sch{\"a}rli, Nathanael and Gehrmann, Sebastian and Tay, Yi and Chung, Hyung Won and Chowdhery, Aakanksha and Le, Quoc and Chi, Ed H and Zhou, Denny and others},
  booktitle={Findings of the Association for Computational Linguistics: ACL 2023},
  pages={13003--13051},
  year={2023}
}

@inproceedings{hendel2023context,
  title={In-context learning creates task vectors},
  author={Hendel, Roee and Geva, Mor and Globerson, Amir},
  booktitle={Findings of the Association for Computational Linguistics: EMNLP 2023},
  pages={9318--9333},
  year={2023}
}

@article{zhao2026understanding,
  title={Understanding parametric and contextual knowledge reconciliation within large language models},
  author={Zhao, Jun and Yang, Yongzhuo and Hu, Xiang and Tong, Jingqi and Lu, Yi and Wu, Wei and Gui, Tao and Zhang, Qi and Huang, Xuanjing},
  journal={Advances in Neural Information Processing Systems},
  volume={38},
  pages={102978--103012},
  year={2026}
}

@article{du2026memory,
  title={Memory for autonomous llm agents: Mechanisms, evaluation, and emerging frontiers},
  author={Du, Pengfei},
  journal={arXiv preprint arXiv:2603.07670},
  year={2026}
}

@inproceedings{salama2025meminsight,
  title={Meminsight: Autonomous memory augmentation for llm agents},
  author={Salama, Rana and Cai, Jason and Yuan, Michelle and Currey, Anna and Sunkara, Monica and Zhang, Yi and Benajiba, Yassine},
  booktitle={Proceedings of the 2025 Conference on Empirical Methods in Natural Language Processing},
  pages={33124--33140},
  year={2025}
}

@article{feder2021causalm,
  title={Causalm: Causal model explanation through counterfactual language models},
  author={Feder, Amir and Oved, Nadav and Shalit, Uri and Reichart, Roi},
  journal={Computational Linguistics},
  volume={47},
  number={2},
  pages={333--386},
  year={2021}
}

@article{lyu2024towards,
  title={Towards faithful model explanation in nlp: A survey},
  author={Lyu, Qing and Apidianaki, Marianna and Callison-Burch, Chris},
  journal={Computational Linguistics},
  volume={50},
  number={2},
  pages={657--723},
  year={2024}
}

@article{ravfogel2411gumbel,
  title={Gumbel counterfactual generation from language models, 2025},
  author={Ravfogel, Shauli and Svete, Anej and Sn{\ae}bjarnarson, V{\'e}steinn and Cotterell, Ryan},
  journal={URL https://arxiv. org/abs/2411.07180}
}

@article{miller2026counterfactual,
  title={Counterfactual reasoning: an analysis of in-context emergence},
  author={Miller, Moritz and Sch{\"o}lkopf, Bernhard and Guo, Siyuan},
  journal={Advances in Neural Information Processing Systems},
  volume={38},
  pages={87510--87544},
  year={2026}
}

@article{pona2026abstract,
  title={Abstract Counterfactuals for Language Model Agents},
  author={Pona, Edoardo and Kazemi, Milad and Du, Yali and Watson, David and Paoletti, Nicola},
  journal={Advances in Neural Information Processing Systems},
  volume={38},
  pages={87484--87509},
  year={2026}
}

@article{yan2023counterfactual,
  title={Counterfactual generation with identifiability guarantees},
  author={Yan, Hanqi and Kong, Lingjing and Gui, Lin and Chi, Yuejie and Xing, Eric and He, Yulan and Zhang, Kun},
  journal={Advances in neural information processing systems},
  volume={36},
  pages={56256--56277},
  year={2023}
}

@article{guo2023causal,
  title={Causal de finetti: On the identification of invariant causal structure in exchangeable data},
  author={Guo, Siyuan and T{\'o}th, Viktor and Sch{\"o}lkopf, Bernhard and Husz{\'a}r, Ferenc},
  journal={Advances in Neural Information Processing Systems},
  volume={36},
  pages={36463--36475},
  year={2023}
}

@inproceedings{nasr2023counterfactual,
  title={Counterfactual identifiability of bijective causal models},
  author={Nasr-Esfahany, Arash and Alizadeh, Mohammad and Shah, Devavrat},
  booktitle={International conference on machine learning},
  pages={25733--25754},
  year={2023},
  organization={PMLR}
}

@article{meng2022locating,
  title={Locating and editing factual associations in gpt},
  author={Meng, Kevin and Bau, David and Andonian, Alex and Belinkov, Yonatan},
  journal={Advances in neural information processing systems},
  volume={35},
  pages={17359--17372},
  year={2022}
}

@inproceedings{geva2023dissecting,
  title={Dissecting recall of factual associations in auto-regressive language models},
  author={Geva, Mor and Bastings, Jasmijn and Filippova, Katja and Globerson, Amir},
  booktitle={Proceedings of the 2023 Conference on Empirical Methods in Natural Language Processing},
  pages={12216--12235},
  year={2023}
}

@article{xie2021explanation,
  title={An explanation of in-context learning as implicit bayesian inference},
  author={Xie, Sang Michael and Raghunathan, Aditi and Liang, Percy and Ma, Tengyu},
  journal={arXiv preprint arXiv:2111.02080},
  year={2021}
}

@article{veitch2021counterfactual,
  title={Counterfactual invariance to spurious correlations: Why and how to pass stress tests},
  author={Veitch, Victor and D'Amour, Alexander and Yadlowsky, Steve and Eisenstein, Jacob},
  journal={arXiv preprint arXiv:2106.00545},
  year={2021}
}

@inproceedings{todd2024function,
  title={Function vectors in large language models},
  author={Todd, Eric and Li, Millicent and Sen Sharma, Arnab and Mueller, Aaron and Wallace, Byron and Bau, David},
  booktitle={International conference on learning representations},
  volume={2024},
  pages={17282--17333},
  year={2024}
}

@article{kahardipraja2026atlas,
  title={The atlas of in-context learning: How attention heads shape in-context retrieval augmentation},
  author={Kahardipraja, Patrick and Achtibat, Reduan and Wiegand, Thomas and Samek, Wojciech and Lapuschkin, Sebastian},
  journal={Advances in Neural Information Processing Systems},
  volume={38},
  pages={118164--118208},
  year={2026}
}

@inproceedings{chen2024inside,
  title={INSIDE: LLMs' internal states retain the power of hallucination detection},
  author={Chen, Chao and Liu, Kai and Chen, Ze and Gu, Yi and Wu, Yue and Tao, Mingyuan and Fu, Zhihang and Ye, Jieping},
  booktitle={International Conference on Learning Representations},
  volume={2024},
  pages={3056--3076},
  year={2024}
}

@article{cohen2024contextcite,
  title={Contextcite: Attributing model generation to context},
  author={Cohen-Wang, Benjamin and Shah, Harshay and Georgiev, Kristian and M{\k{a}}dry, Aleksander},
  journal={Advances in Neural Information Processing Systems},
  volume={37},
  pages={95764--95807},
  year={2024}
}

@article{zou2023representation,
  title={Representation engineering: A top-down approach to ai transparency},
  author={Zou, Andy and Phan, Long and Chen, Sarah and Campbell, James and Guo, Phillip and Ren, Richard and Pan, Alexander and Yin, Xuwang and Mazeika, Mantas and Dombrowski, Ann-Kathrin and others},
  journal={arXiv preprint arXiv:2310.01405},
  year={2023}
}

@article{turner2023steering,
  title={Steering language models with activation engineering},
  author={Turner, Alexander Matt and Thiergart, Lisa and Leech, Gavin and Udell, David and Vazquez, Juan J and Mini, Ulisse and MacDiarmid, Monte},
  journal={arXiv preprint arXiv:2308.10248},
  year={2023}
}

@article{shai2024transformers,
  title={Transformers represent belief state geometry in their residual stream},
  author={Shai, Adam S and Marzen, Sarah E and Teixeira, Lucas and Oldenziel, Alexander G and Riechers, Paul M},
  journal={Advances in Neural Information Processing Systems},
  volume={37},
  pages={75012--75034},
  year={2024}
}

@inproceedings{bigelow2024context,
  title={In-context learning dynamics with random binary sequences},
  author={Bigelow, Eric and Lubana, Ekdeep Singh and Dick, Robert and Tanaka, Hidenori and Ullman, Tomer},
  booktitle={International Conference on Learning Representations},
  volume={2024},
  pages={56330--56373},
  year={2024}
}

@incollection{balke2022probabilistic,
  title={Probabilistic evaluation of counterfactual queries},
  author={Balke, Alexander and Pearl, Judea},
  booktitle={Probabilistic and causal inference: The works of Judea Pearl},
  pages={237--254},
  year={2022}
}

@article{shpitser2012effects,
  title={Effects of treatment on the treated: Identification and generalization},
  author={Shpitser, Ilya and Pearl, Judea},
  journal={arXiv preprint arXiv:1205.2615},
  year={2012}
}

@article{pearl2003causality,
  title={Causality: models, reasoning, and inference},
  author={Pearl, Judea and others},
  journal={Econometric Theory},
  volume={19},
  number={675-685},
  pages={46},
  year={2003},
  publisher={Cambridge university press}
}

@article{brown2020language,
  title={Language models are few-shot learners},
  author={Brown, Tom and Mann, Benjamin and Ryder, Nick and Subbiah, Melanie and Kaplan, Jared D and Dhariwal, Prafulla and Neelakantan, Arvind and Shyam, Pranav and Sastry, Girish and Askell, Amanda and others},
  journal={Advances in neural information processing systems},
  volume={33},
  pages={1877--1901},
  year={2020}
}

@article{kaplan2020scaling,
  title={Scaling laws for neural language models},
  author={Kaplan, Jared and McCandlish, Sam and Henighan, Tom and Brown, Tom B and Chess, Benjamin and Child, Rewon and Gray, Scott and Radford, Alec and Wu, Jeffrey and Amodei, Dario},
  journal={arXiv preprint arXiv:2001.08361},
  year={2020}
}

@article{bommasani2021opportunities,
  title={On the opportunities and risks of foundation models},
  author={Bommasani, Rishi and Hudson, Drew A and Adeli, Ehsan and Altman, Russ and Arora, Simran and von Arx, Sydney and Bernstein, Michael S and Bohg, Jeannette and Bosselut, Antoine and Brunskill, Emma and others},
  journal={arXiv preprint arXiv:2108.07258},
  year={2021}
}

@inproceedings{niu2022efficient,
  title={Efficient test-time model adaptation without forgetting},
  author={Niu, Shuaicheng and Wu, Jiaxiang and Zhang, Yifan and Chen, Yaofo and Zheng, Shijian and Zhao, Peilin and Tan, Mingkui},
  booktitle={International conference on machine learning},
  pages={16888--16905},
  year={2022},
  organization={PMLR}
}

@inproceedings{chen2026test,
  title={Test-time adaptation for llm agents via environment interaction},
  author={Chen, Arthur and Liu, Zuxin and Zhang, Jianguo and Prabhakar, Akshara and Liu, Zhiwei and Heinecke, Shelby and Savarese, Silvio and Zhong, Victor and Xiong, Caiming},
  booktitle={International Conference on Learning Representations},
  volume={2026},
  pages={39256--39288},
  year={2026}
}

@article{abel2023definition,
  title={A definition of continual reinforcement learning},
  author={Abel, David and Barreto, Andr{\'e} and Van Roy, Benjamin and Precup, Doina and Van Hasselt, Hado and Singh, Satinder},
  journal={Advances in Neural Information Processing Systems},
  volume={36},
  pages={50377--50407},
  year={2023}
}

@article{lu2020reset,
  title={Reset-free lifelong learning with skill-space planning},
  author={Lu, Kevin and Grover, Aditya and Abbeel, Pieter and Mordatch, Igor},
  journal={arXiv preprint arXiv:2012.03548},
  year={2020}
}

@inproceedings{zhong2024memorybank,
  title={Memorybank: Enhancing large language models with long-term memory},
  author={Zhong, Wanjun and Guo, Lianghong and Gao, Qiqi and Ye, He and Wang, Yanlin},
  booktitle={Proceedings of the AAAI conference on artificial intelligence},
  volume={38},
  number={17},
  pages={19724--19731},
  year={2024}
}

@article{mi2026skill,
  title={Skill-pro: Learning reusable skills from experience via non-parametric ppo for llm agents},
  author={Mi, Qirui and Ma, Zhijian and Yang, Mengyue and Li, Haoxuan and Wang, Yisen and Zhang, Haifeng and Wang, Jun},
  journal={arXiv preprint arXiv:2602.01869},
  year={2026}
}

@article{yao2026can,
  title={Can Dependencies Induced by LLM-Agent Workflows Be Trusted?},
  author={Yao, Yu and Song, Yiliao Lia and Xie, Yian and Fan, Mengdan and Guo, Mingyu and Liu, Tongliang},
  journal={Advances in Neural Information Processing Systems},
  volume={38},
  pages={17431--17469},
  year={2026}
}

@inproceedings{finn2017model,
  title={Model-agnostic meta-learning for fast adaptation of deep networks},
  author={Finn, Chelsea and Abbeel, Pieter and Levine, Sergey},
  booktitle={International conference on machine learning},
  pages={1126--1135},
  year={2017},
  organization={PMLR}
}

@inproceedings{cobbe2019quantifying,
  title={Quantifying generalization in reinforcement learning},
  author={Cobbe, Karl and Klimov, Oleg and Hesse, Chris and Kim, Taehoon and Schulman, John},
  booktitle={International conference on machine learning},
  pages={1282--1289},
  year={2019},
  organization={PMLR}
}

@article{liao2026belief,
  title={Belief Memory: Agent Memory Under Partial Observability},
  author={Liao, Junfeng and Wang, Qizhou and Zhu, Jianing and Du, Bo and Yan, Rui and Chen, Xiuying},
  journal={arXiv preprint arXiv:2605.05583},
  year={2026}
}

@article{du2023shortcut,
  title={Shortcut learning of large language models in natural language understanding},
  author={Du, Mengnan and He, Fengxiang and Zou, Na and Tao, Dacheng and Hu, Xia},
  journal={Communications of the ACM},
  volume={67},
  number={1},
  pages={110--120},
  year={2023},
  publisher={ACM New York, NY, USA}
}

@article{amodei2016concrete,
  title={Concrete problems in AI safety},
  author={Amodei, Dario and Olah, Chris and Steinhardt, Jacob and Christiano, Paul and Schulman, John and Man{\'e}, Dan},
  journal={arXiv preprint arXiv:1606.06565},
  year={2016}
}

@inproceedings{panwar2024context,
  title={In-context learning through the bayesian prism},
  author={Panwar, Madhur and Ahuja, Kabir and Goyal, Navin},
  booktitle={International Conference on Learning Representations},
  volume={2024},
  pages={49789--49843},
  year={2024}
}

@article{kirk2021survey,
  title={A survey of generalisation in deep reinforcement learning},
  author={Kirk, Robert and Zhang, Amy and Grefenstette, Edward and Rockt{\"a}schel, Tim},
  journal={arXiv preprint arXiv:2111.09794},
  volume={1},
  number={16},
  pages={3},
  year={2021}
}

@inproceedings{le2026sage,
  title={SAGE: Scalable AI Governance \& Evaluation},
  author={Le, Benjamin H and Lu, Xueying and Stern, Nicholas and Liu, Wenqiong and Lapchuk, Igor and Li, Xiang and Zheng, Baofen and Rosenberg, Kevin and Huang, Jiewen and Zhang, Zhe and others},
  booktitle={Proceedings of the 32nd ACM SIGKDD Conference on Knowledge Discovery and Data Mining V. 2},
  pages={7555--7565},
  year={2026}
}

@article{su2024texttt,
  title={ConflictBank: A Benchmark for Evaluating the Influence of Knowledge Conflicts in LLMs},
  author={Su, Zhaochen and Zhang, Jun and Qu, Xiaoye and Zhu, Tong and Li, Yanshu and Sun, Jiashuo and Li, Juntao and Zhang, Min and Cheng, Yu},
  journal={Advances in Neural Information Processing Systems},
  volume={37},
  pages={103242--103268},
  year={2024}
}

@article{hou2024wikicontradict,
  title={Wikicontradict: A benchmark for evaluating llms on real-world knowledge conflicts from wikipedia},
  author={Hou, Yufang and Pascale, Alessandra and Carnerero-Cano, Javier and Tchrakian, Tigran and Marinescu, Radu and Daly, Elizabeth and Padhi, Inkit and Sattigeri, Prasanna},
  journal={Advances in Neural Information Processing Systems},
  volume={37},
  pages={109701--109747},
  year={2024}
}

@inproceedings{orgad2025llms,
  title={Llms know more than they show: On the intrinsic representation of llm hallucinations},
  author={Orgad, Hadas and Toker, Michael and Gekhman, Zorik and Reichart, Roi and Szpektor, Idan and Kotek, Hadas and Belinkov, Yonatan},
  booktitle={International Conference on Learning Representations},
  volume={2025},
  pages={66880--66913},
  year={2025}
}

@inproceedings{azaria2023internal,
  title={The internal state of an LLM knows when it’s lying},
  author={Azaria, Amos and Mitchell, Tom},
  booktitle={Findings of the Association for Computational Linguistics: EMNLP 2023},
  pages={967--976},
  year={2023}
}

@article{kadavath2022language,
  title={Language models (mostly) know what they know},
  author={Kadavath, Saurav and Conerly, Tom and Askell, Amanda and Henighan, Tom and Drain, Dawn and Perez, Ethan and Schiefer, Nicholas and Hatfield-Dodds, Zac and DasSarma, Nova and Tran-Johnson, Eli and others},
  journal={arXiv preprint arXiv:2207.05221},
  year={2022}
}

@inproceedings{cheang2026llms,
  title={Do LLMs Really Know What They Don’t Know? Internal States Mainly Reflect Knowledge Recall Rather Than Truthfulness},
  author={Cheang, Chi Seng and Chan, Hou Pong and Zhang, Wenxuan and Deng, Yang},
  booktitle={Findings of the Association for Computational Linguistics: ACL 2026},
  pages={713--730},
  year={2026}
}

@article{li2024survey,
  title={A survey on the honesty of large language models},
  author={Li, Siheng and Yang, Cheng and Wu, Taiqiang and Shi, Chufan and Zhang, Yuji and Zhu, Xinyu and Cheng, Zesen and Cai, Deng and Yu, Mo and Liu, Lemao and others},
  journal={arXiv preprint arXiv:2409.18786},
  year={2024}
}

@inproceedings{creager2021environment,
  title={Environment inference for invariant learning},
  author={Creager, Elliot and Jacobsen, J{\"o}rn-Henrik and Zemel, Richard},
  booktitle={International conference on machine learning},
  pages={2189--2200},
  year={2021},
  organization={PMLR}
}

@article{yao2022react,
  title={React: Synergizing reasoning and acting in language models},
  author={Yao, Shunyu and Zhao, Jeffrey and Yu, Dian and Du, Nan and Shafran, Izhak and Narasimhan, Karthik and Cao, Yuan},
  journal={arXiv preprint arXiv:2210.03629},
  year={2022}
}

@inproceedings{liu2024agentbench,
  title={Agentbench: Evaluating llms as agents},
  author={Liu, Xiao and Yu, Hao and Zhang, Hanchen and Xu, Yifan and Lei, Xuanyu and Lai, Hanyu and Gu, Yu and Ding, Hangliang and Men, Kaiwen and Yang, Kejuan and others},
  booktitle={International Conference on Learning Representations},
  volume={2024},
  pages={52989--53046},
  year={2024}
}

@inproceedings{mialon2024gaia,
  title={Gaia: a benchmark for general ai assistants},
  author={Mialon, Gr{\'e}goire and Fourrier, Cl{\'e}mentine and Wolf, Thomas and LeCun, Yann and Scialom, Thomas},
  booktitle={International Conference on Learning Representations},
  volume={2024},
  pages={9025--9049},
  year={2024}
}

@inproceedings{qin2024toolllm,
  title={Toolllm: Facilitating large language models to master 16000+ real-world apis},
  author={Qin, Yujia and Liang, Shihao and Ye, Yining and Zhu, Kunlun and Yan, Lan and Lu, Yaxi and Lin, Yankai and Cong, Xin and Tang, Xiangru and Qian, Bill and others},
  booktitle={International Conference on Learning Representations},
  volume={2024},
  pages={9695--9717},
  year={2024}
}

@inproceedings{yu2026polyskill,
  title={Polyskill: Learning generalizable skills through polymorphic abstraction for continual learning},
  author={Yu, Simon and Li, Gang and Shi, Weiyan and Qi, Peng},
  booktitle={International Conference on Learning Representations},
  volume={2026},
  pages={140298--140326},
  year={2026}
}

@article{zheng2025skillweaver,
  title={Skillweaver: Web agents can self-improve by discovering and honing skills},
  author={Zheng, Boyuan and Fatemi, Michael Y and Jin, Xiaolong and Wang, Zora Zhiruo and Gandhi, Apurva and Song, Yueqi and Gu, Yu and Srinivasa, Jayanth and Liu, Gaowen and Neubig, Graham and others},
  journal={arXiv preprint arXiv:2504.07079},
  year={2025}
}

@article{wu2026co,
  title={Co-evolving llm decision and skill bank agents for long-horizon tasks},
  author={Wu, Xiyang and Li, Zongxia and Shi, Guangyao and Duffy, Alexander and Marques, Tyler and Olson, Matthew Lyle and Zhou, Tianyi and Manocha, Dinesh},
  journal={arXiv preprint arXiv:2604.20987},
  year={2026}
}

@article{zhang2026coevoskills,
  title={Coevoskills: Self-evolving agent skills via co-evolutionary verification},
  author={Zhang, Hanrong and Fan, Shicheng and Zou, Henry Peng and Chen, Yankai and Wang, Zhenting and Zhou, Jiayu and Li, Chengze and Huang, Wei-Chieh and Yao, Yifei and Zheng, Kening and others},
  journal={arXiv preprint arXiv:2604.01687},
  year={2026}
}

@article{zhao2026large,
  title={Large Language Model Agents Are Not Always Faithful Self-Evolvers. arXiv 2026},
  author={Zhao, Weixiang and Wang, Yingshuo and Zhang, Yichen and Deng, Yang and Zhao, Yanyan and Che, Wanxiang and Qin, Bing and Liu, Ting},
  journal={arXiv preprint arXiv:2601.22436},
  year={2026}
}

@article{hu2026continual,
  title={When continual learning moves to memory: A study of experience reuse in llm agents},
  author={Hu, Qisheng and Long, Quanyu and Wang, Wenya},
  journal={arXiv preprint arXiv:2604.27003},
  year={2026}
}

\newpage
\appendix
\addcontentsline{toc}{section}{Appendix} 
\begin{center}
	\rule{5.6in}{1.2pt}\\
    \vspace{2mm}
	{\Large\bf Appendix for ``Beyond the Shadows of Plato's Cave:\\Evaluating False Memory in Autonomous Agents via Counterfactual Reasoning''}
	\rule{5.6in}{1.2pt}
\end{center}

\part{} 

\parttoc 

\section{Related Work}
\paragraph{Autonomous agents.}
Autonomous agents have emerged as an important mechanism for addressing complex and rapidly changing problems.
Thanks to their self-improvement capabilities, these agents can iteratively refine their behavior based on feedback from previous attempts \citep{wang2023voyager, shinn2023reflexion, madaan2023self}, enabling them to achieve strong performance across a wide range of tasks \cite{yao2022react, liu2024agentbench, mialon2024gaia, qin2024toolllm}. 

This iterative learning process has consequently made memory a standard component of autonomous agents \citep{du2026memory, salama2025meminsight, packer2023memgpt, wang2024agent, xu2026mem, zhang2026g}. 
For example, many works leverage memory to retain feedback collected during self-improvement \citep{fu2025agentrefine}, while subsequent works use memory as a source of experience from which reusable artifacts can be distilled, including functional code \cite{lingam2025enhancing}, APIs \cite{zheng2025skillweaver}, skills \cite{yu2026polyskill, wu2026co, zhang2026coevoskills}, and abstract knowledge \citep{yang2024buffer, zhao2024expel}. 

Such accumulated experiences can then be retrieved and adapted to facilitate problem solving in similar but potentially different settings \citep{tran2026bifrost}. 
However, effective adaptation requires determining whether a stored experience remains faithful under changes in context and environment \citep{zhao2026large, hu2026continual}. 
This raises a critical yet underexplored problem of evaluating false memories studied by this work.
Addressing the false-memory problem is therefore essential for ensuring that autonomous agents can adapt effectively and reliably.

\paragraph{Agentic memory.}
Existing work has demonstrated that leveraging memory can provide substantial benefits to autonomous agents \citep{shutova2026evaluating}, yet agentic memories are inherently vulnerable to degradation \citep{chen2024agentpoison, karamchandani2026your}. 
Prior studies have shown that memories can be affected by shortcut learning \citep{song2024shortcut, tang2023large, mirzadeh2025gsm, mccoy2019right, veitch2021counterfactual}, distribution shifts \cite{amodei2016concrete, cobbe2019quantifying}, and conflicts with newly acquired knowledge \citep{misra2026gitchameleon, zhao2026understanding, xie2024adaptive, xu2024knowledge, zhao2024analysing, jin2024tug}. 
However, these works largely frame false memory as a problem of generalization \citep{da2026mitigating}, proposing benchmarks to assess retrieval effectiveness \citep{hu2026evaluating, maharana2024evaluating, wu2024longmemeval} and focusing primarily on calibrating reward design \citep{jin2024tug} or fine-tuning strategies \citep{mccoy2019right} to mitigate such failures. 

More recently, Misevolution \citep{shao2026your} provides empirical evidence that memory can degrade during self-evolution, even in large-scale agents. Demonstrations may exploit reward hacking and induce shortcut learning, while curated tools can become environment-dependent. Crucially, seemingly benign workflows may still introduce safety risks, suggesting that memory degradation is not always apparent from surface-level behavior. 
While these observations are important, they show empirical evidence of memory degradation within self-evolution rather than evaluating how faithful the memory can be under different changes of conditions.

This motivates the need for a formal evaluation framework that can assess false memory across diverse and arbitrary scenarios in practice, which is the focus of our work.

\paragraph{Counterfactual reasoning.}
Counterfactual reasoning has traditionally been grounded in structural causal models, where hypothetical interventions are used to characterize how changes to a variable affect downstream outcomes \citep{pearl2009causal}. 
Recent work has extended these ideas to large language models \citep{ravfogel2411gumbel}, studying both their ability to perform causal and counterfactual reasoning and the use of counterfactual interventions to understand model behavior \citep{feder2021causalm, lyu2024towards}. 
Notably, previous works demonstrate that counterfactual reasoning naturally emerges in in-context learning \citep{miller2026counterfactual} and subsequently provide theoretical guarantees of its identifiability \citep{yan2023counterfactual, guo2023causal, nasr2023counterfactual}. Subsequent works study the abstraction of knowledge when performing counterfactual reasoning \citep{pona2026abstract}.

These approaches primarily focus on explaining or evaluating model behavior based on answer effectiveness, which requires sampling answers and consequently incurs additional generation costs.
In contrast, we focus on directly studying agent internal beliefs by leveraging counterfactual reasoning to evaluate false memory without answer generation.

\paragraph{Concept representation in hidden states.}
A growing body of work demonstrates that agent hidden states encode compact abstractions that mediate their predictive behavior. 
Studies of in-context learning suggest that a set of demonstrations can represent a latent task vector \citep{hendel2023context, meng2022locating, geva2023dissecting, kahardipraja2026atlas}. These findings motivate using internal activations as a representation of the concepts underlying agent beliefs. 

Rather than using representations only to localize or explain existing model behavior, we track how this representation changes when the agent encounters a counterfactual query. This enables us to characterize the agent memory-induced belief through concept drift in representation space and to evaluate whether the resulting change reflects appropriate adaptation or persistence of a false memory.

\section{Theoretical Analysis}
\label{app:theoretical_analysis}
\subsection{Proof of Proposition~\ref{prop:cf_identifiability}}
\label{app:theoretical_analysis:proof_cf_identifiability}

\begin{proof}
Assume the SCM is Markovian, i.e., the exogenous variables $U_M,U_{\Theta_M},U_Q,U'$ are mutually independent. For any value $m$ of $M$, we have:

\paragraph{(i) Reduction of the counterfactual conditional.}
Since $M$ is a non-descendant of $Q$, intervening on $Q$ does not change $M$, i.e., $M_{q'}=M$.
Hence, by the counterfactual reduction rule for pretreatment evidence~\citep{pearl2003causality},
\begin{equation}
    P(\Theta'_{q'}\mid M=m) = P(\Theta'\mid do(Q=q'),M=m).
\end{equation}

\paragraph{(ii) Exchange intervention for observation.}
Consider the graph $G_{\underline Q}$ obtained by deleting all outgoing edges from $Q$. 
Because $Q$ is a root variable and, by the Markovian assumption, has no unobserved common cause with $M$, $\Theta_M$, or $\Theta'$, $Q$ is d-separated from $\Theta'$ given $M$ in $G_{\underline Q}$:
\begin{equation}
    (\Theta'\perp Q\mid M)_{G_{\underline Q}}.   
\end{equation}
Therefore, by Rule 2 of do-calculus,
\begin{equation}
    P(\Theta'\mid do(Q=q'),M=m) = P(\Theta'\mid Q=q',M=m).   
\end{equation}

Combining (i) and (ii) yields
\begin{equation}
    \boxed{
    P(\Theta'_{Q=q'}\mid M=m) = P(\Theta'\mid   do(Q=q'),M=m) = P(\Theta'\mid Q=q',M=m).
    }
\end{equation}
Therefore, the counterfactual distribution of $\Theta'$ under $do(Q=q')$ is identifiable from the observational distribution.  
\end{proof}

  
  


\subsection{Proof of Theorem~\ref{thm:concept_drift_identifiability_cf}}
\label{app:theoretical_analysis:proof_concept_drift_identifiability_cf}
\newsavebox{\twinnetworkbox}
\savebox{\twinnetworkbox}{%
  \begin{tikzpicture}[
      var/.style = {draw, circle, minimum size=0.85cm, inner sep=0pt, font=\small},
      exo/.style = {var, dashed},                 
      arr/.style = {-{Stealth[length=1.2mm]}, thick},
      lbl/.style = {font=\small, inner sep=1pt}
  ]
  \node[exo] (UM) {$U_M$};
  \node[var, below=0.5cm of UM]  (M)      {$M$};
  \node[var, below=0.5cm of M]   (ThetaM) {$\Theta_M$};
  \node[exo, left=0.6cm of ThetaM] (UTheta) {$U_\Theta$};

  \node[var, below left=0.7cm and 0.4cm of ThetaM]  (Theta1) {$\Theta'$};
  \node[var, below right=0.7cm and 0.4cm of ThetaM] (Theta2) {$\Theta'^{*}$};

  \node[var, left=0.6cm of Theta1] (Q)  {$Q$};
  \node[exo, left=0.6cm of Q]      (UQ) {$U_Q$};

  \node[var, right=0.6cm of Theta2,
        label={[lbl]right:$=q'$}] (Qstar) {$Q^{*}$};

  \node[exo, below=1.1cm of ThetaM] (Uprime) {$U'$};


  \foreach \s/\t in {UM/M, M/ThetaM, UTheta/ThetaM,
                     ThetaM/Theta1, ThetaM/Theta2,
                     UQ/Q, Q/Theta1, Qstar/Theta2,
                     Uprime/Theta1, Uprime/Theta2}
    \draw[arr] (\s) -- (\t);
  \end{tikzpicture}%
}

\paragraph{Twin-network construction}
Following~\cite{balke2022probabilistic}, a twin network is a graph $G_{\text{twin}}$ over two copies of $\mathbf{V}$ — the factual world and the counterfactual world $\mathbf{V}^*$ — which share the same exogenous variables $\mathbf{U}$.

By the node-merging rule~\citep{shpitser2012effects}, a counterfactual node whose parents are identical to its factual counterpart's can be merged with it. As a result, the twin network of the belief update is presented in Fig.~\ref{fig:twin_network}, $M$ and $\Theta_M$ are non-descendants of $Q$, so $M^* \equiv M$ and $\Theta_M^* \equiv \Theta_M$. Only $\Theta'$ gets a distinct copy.
\begin{figure}
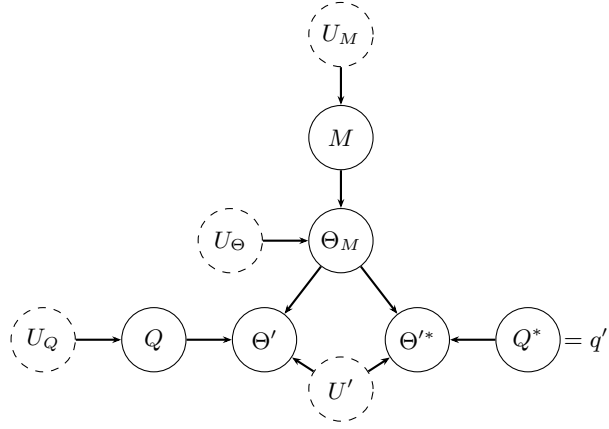

  \centering
  \usebox{\twinnetworkbox}
  \caption{Twin network for the belief update mechanism.}
  \label{fig:twin_network}
\end{figure}

\paragraph{Concept-drift identifiability via counterfactual.}
We first establish some theoretical results.

\begin{lemma}[Counterfactual Ignorability]
\label{lemma:cf_ignorability}
    In the twin network (Fig.~\ref{fig:twin_network}), $\Theta'_{q'} \indep Q \mid M$.
\end{lemma}

\begin{proof}
    Enumerate every path between $Q$ and $\Theta'^*$ in Fig.~\ref{fig:twin_network} ($Q^*$ has no parents and is a constant, so it creates no path):
    \begin{enumerate}[label=(\roman*)]
        \item $Q \to \Theta' \leftarrow \Theta_M \to \Theta'^*$. Here $\Theta'$ is a collider, and neither it nor any descendant is in the conditioning set ${M}$. The path is blocked.
        \item $Q \to \Theta' \leftarrow U' \to \Theta'^*$. Again, $\Theta'$ is an unconditioned collider. The path is blocked.
        \item $Q \leftarrow U_Q$. $U_Q$ has no other children, so this is a dead end.
    \end{enumerate}
    
    Consequently, every path is blocked, so $Q$ and $\Theta'^*$ are d-separated given $M$. Under a mild Markovianity assumption that $U$ variables are independent, the twin network satisfies the global Markov property, so d-separation implies conditional independence.
\end{proof}

\begin{lemma}[Consistency]
\label{lemma:consistency}
    If $Q = q$, then $\Theta'_q = \Theta'$.
\end{lemma}

\begin{proof}
    $\Theta'_q = f'(\Theta_M, q, U') = f'(\Theta_M, Q, U') = \Theta'$ on the event ${Q=q}$.
\end{proof}

Given the results of Lemma~\ref{lemma:cf_ignorability} and Lemma~\ref{lemma:consistency}, we start the proof of Thm.~\ref{thm:concept_drift_identifiability_cf} by first extending Prop.~\ref{prop:cf_identifiability} to the twin-network as follows.

\begin{proposition}[Counterfactual-Factual Identifiability]
\label{prop:f_cf_identifiability}
    Given $q \in \{q_0, q'\}$,
    \begin{equation}
        P(\Theta'_q = \theta \mid Q=q_0, M=m) = P(\Theta' = \theta \mid Q=q, M=m).  
    \end{equation}
\end{proposition}

\begin{proof}
    We derive the proof of the factual term and the counterfactual term as follows.
    
    \paragraph{Factual term ($q=q_0$).}
    By Lemma~\ref{lemma:consistency}, on ${Q=q_0}$ we have $\Theta'_{q_0} = \Theta'$. Hence, 
    \begin{equation}
        P(\Theta'_{q_0}=\theta \mid q_0, m) = P(\Theta'=\theta \mid q_0, m).
    \end{equation}

    \paragraph{Counterfactual term ($q=q'$).}
    Apply Lemma~\ref{lemma:cf_ignorability} (L1) and Lemma~\ref{lemma:consistency} (L2),
    \begin{align}  
        P(\Theta'_{q'}=\theta \mid Q=q_0, M=m)  
        &\overset{\text{L1}}{=} P(\Theta'_{q'}=\theta \mid M=m) \\  
        &\overset{\text{L1, positivity}}{=} P(\Theta'_{q'}=\theta \mid Q=q', M=m) \\  
        &\overset{\text{L2}}{=} P(\Theta'=\theta \mid Q=q', M=m).  
    \end{align}
    The middle step conditions on $Q=q'$, which requires positivity.
\end{proof}

Now we prove the main result of Thm.~\ref{thm:concept_drift_identifiability_cf}.
\begin{proof}

    By~\cite{shpitser2012effects}, an Effect of Treatment on the Treated (ETT) is defined as
    \begin{equation}
        ETT = \mathbb{E} \big[ \gamma(\Theta'_{q'}) - \gamma(\Theta'_{q_0}) \mid Q=q_0, M=m \big].
    \end{equation}
    By linearity of conditional expectation, we have
    \begin{equation}
        ETT = \mathbb{E} \big[ \gamma(\Theta'_{q'}) - \gamma(\Theta'_{q_0}) \mid Q=q_0, M=m \big] = \mathbb{E}[\gamma(\Theta'_{q'}) \mid q_0, m] - \mathbb{E}[\gamma(\Theta'_{q_0}) \mid q_0, m].
    \end{equation}
    The second term is $\mathbb{E}[\gamma(\Theta') \mid q_0, m]$ by consistency (Lemma~\ref{lemma:consistency}).
    The first term 
    \begin{equation}
    \begin{split}
        \mathbb{E}[\gamma(\Theta'_{q'}) \mid q_0, m] &= \sum_\theta \gamma(\theta) P(\Theta'_{q'
        }=\theta \mid q_0, m) \\
        &\overset{\text{Prop.~\ref{prop:f_cf_identifiability}}}{=}
        \sum_\theta \gamma(\theta) P(\Theta'=\theta \mid q', m) \\
        &= \mathbb{E}[\gamma(\Theta') \mid q', m].
    \end{split}
    \end{equation}

    Therefore,
    \begin{equation}
        ETT = \mathbb{E}[\gamma(\Theta') \mid q', m] - \mathbb{E}[\gamma(\Theta') \mid q_0, m].
    \end{equation}

    Let $\gamma_h(P) = \mathbb{E}_P\gamma(\Theta')$,
    \begin{equation}
    \boxed{
        ETT = \Delta(m, q', q_0) = \gamma_h(P(\Theta' \mid m, q')) - \gamma_h(P(\Theta' \mid m, q_0)).   
    }
    \end{equation}
\end{proof}

\subsection{Proof of Theorem~\ref{thm:concept_approximation_hiddenstate}}
\label{app:theoretical_analysis:concept_approx_hiddenstate_proof}
\begin{proof}
Let $M=\{(q_i,a_i)\}_{i=1}^{N}$ denote the demonstrations and let $\theta_M$ denote the consistent concept represented by $M$. We define the predictive concept associated with $\theta_M$ through its induced predictive distribution $p(\cdot\mid\theta_M,q)$, for any query $q$.

We first formalize the two approximation errors in the theorem.

\paragraph{Step 1: Representative demonstrations.}
Because $M$ sufficiently represents the consistent concept $\theta_M$, its induced predictive distribution approximates the predictive distribution of the underlying concept. Hence, there exists $\epsilon_M\geq0$ such that 
\begin{equation}
\label{eq:mem_eps}
    d\!\left(  
    p(\cdot\mid M,q),  
    p(\cdot\mid\theta_M,q)  
    \right)  
    \leq \epsilon_M,      
\end{equation}
for the queries $q$ under consideration, where $d(\cdot,\cdot)$ is an appropriate distributional distance. The quantity $\epsilon_M$ captures the error arising from finite or imperfectly representative demonstrations.

\paragraph{Step 2: Predictive sufficiency of the hidden state.}
Consider the residual-stream hidden state $h(M+q)$ at layer $l$ and at the answer-commit position. Since the subsequent transformer layers and the output unembedding map this internal representation to the model predictive distribution, the hidden state provides a sufficient representation of the information used for prediction, up to the approximation associated with the selected layer. Thus, there exists a decoder $D_l$ and $\epsilon_h\geq0$ such that  
\begin{equation}
\label{eq:hiddenstate_eps}
    d\!\left(  
    D_l(h(M+q)),  
    p(\cdot\mid M,q)  
    \right)  
    \leq \epsilon_h. 
\end{equation}  
Here, $\epsilon_h$ measures the information lost when the complete model computation is represented by the hidden state at layer $l$.

\paragraph{Step 3: Combine the two approximations.}
Applying the triangle inequality to ~\eqref{eq:mem_eps} and~\eqref{eq:hiddenstate_eps} gives  
\begin{align}
    d\!\left(  
D_l(h(M+q)),  
p(\cdot\mid\theta_M,q)  
\right)  
&\leq  
d\!\left(  
D_l(h(M+q)),  
p(\cdot\mid M,q)  
\right)
\\& \quad +   
d\!\left(  
p(\cdot\mid M,q),  
p(\cdot\mid\theta_M,q)  
\right)  
\\  
&\leq  
\epsilon_h+\epsilon_M. 
\end{align}  
Therefore,  
\begin{equation}
    \boxed{  
    d\!\left(  
    D_l(h(M+q)),  
    p(\cdot\mid\theta_M,q)  
    \right)  
    \leq  
    \epsilon_M+\epsilon_h.  
    }  
\end{equation}

Thus, the answer-commit hidden state contains sufficient information to recover the predictive behavior associated with the concept $\theta_M$, up to the combined error of demonstration representativeness and hidden-state sufficiency. In particular, when both errors vanish,  
$\epsilon_M,\epsilon_h\rightarrow0$, we obtain $D_l(h(M+q)) \rightarrow p(\cdot\mid\theta_M,q)$.
Hence, the residual-stream hidden state $h(M+q)$ provides an approximate, observable representation of the predictive concept induced by the in-context demonstrations.
\end{proof}

\subsection{Proof of Theorem~\ref{thm:answer_free}}
\label{app:theoretical_analysis:answer_free_proof}
\begin{proof}
By Asm.~\ref{asm:answer_readout}, for any prompt $p$,
\begin{equation}
    \ell(p) = h(p)^\top \bar{g}_O + c_O,
\end{equation}
where $\bar{g}_O$ and $c_O$ are fixed for the binary objective.

Therefore, for each leave-one-out memory example,  
\begin{equation}
    \ell(M_{-i}+q')
    =
    h(M_{-i}+q')^\top\bar g_O + c_O,
\end{equation}
and
\begin{equation}
    \ell(M_{-i}+q_i)
    =
    h(M_{-i}+q_i)^\top\bar g_O + c_O.
\end{equation}
Subtracting the two equations cancels out $c_O$, and taking the expectation over $i$ gives
\begin{align}
    \Delta_\ell &= \mathbb{E}_{i} \left[ \left( 
        h(M_{-i} + q') - h(M_{-i} + q_i)
    \right)^\top \bar{g}_O \right] \\
    &= \langle \mathbb{E}_{i} \left[
        h(M_{-i} + q') - h(M_{-i} + q_i) \right] , \bar{g}_O \rangle \\
        &= \boxed{\langle \Delta_{\theta}, \bar{g}_O \rangle}.
\end{align}
Therefore, the change in answer preference is the projection of the concept drift onto the answer direction.
\end{proof}

\section{Experiments}
\subsection{Layer Choices of Concept Extraction}
\label{app:experiments:layer_choice}
\begin{wrapfigure}{r}{0.4\textwidth}
    \vspace{-1mm}
    \centering
    \includegraphics[width=\linewidth]{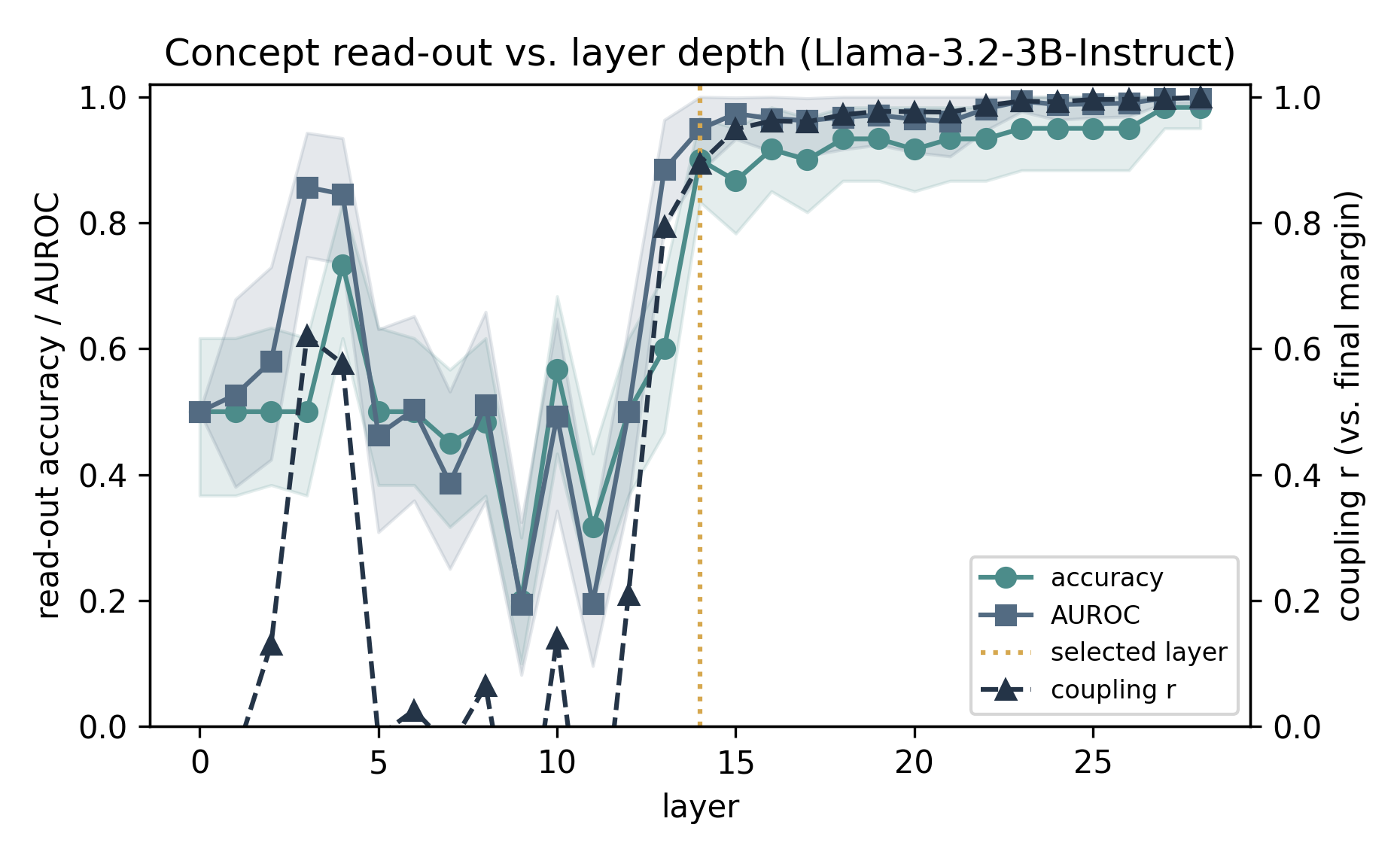}
    \vspace{-9mm}
    \caption{\small Layer sweep results of Llama-3.2-3B-Instruct.}
    \label{fig:layer_sweep}
    \vspace{-6mm}
\end{wrapfigure}
We follow the exact layer position as stated in previous well-established works \citep{tran2026bifrost, zou2023representation, turner2023steering}, i.e., the 14th layer in Llama-3.2-3B-Instruct. These works demonstrate that mid-level layers usually capture high-level conceptual knowledge. 

To further justify the choice of layer position in our work, we conduct a layer sweep on the Tomatoes dataset to observe which layers are representative enough to extract the concept. Specifically, we feed demonstrations from the memory as in-context examples, subsequently extract hidden states across layers, and evaluate whether they capture the concept to reflect the final answer. Fig.~\ref{fig:layer_sweep} shows the sweep results, measured by three metrics: accuracy, AUROCs, and coupling r, which justifies that the selected layer is representative of the concept.

\subsection{Experimental Setup}
\label{app:experiments:setup}
We evaluate FAME on three domains that represent false memory practice, including math problem solving, code generation with version maintenance, and complex reasoning. We then release corresponding counterfactual queries for false-memory evaluation on such domains.

\paragraph{GSM-Symbolic.}
Agents solving math usually encounter false memory, e.g., an agent familiar with ``in total'' as a cue for addition, while it could induce other operations depending on contexts.
We thus review the existing GSM-Symbolic \citep{mirzadeh2025gsm}, and subsequently summarize 40 templates that easily suffer from false memory to evaluate spurious correlation and environment shift. Each template is constructed as a memory with 4 examples. Counterfactual queries either keep the spurious feature but change the underlying arithmetic operations (spurious keeping), or remove such features but keep the memory operations (spurious removal). For environment shift, operations remain the same but numerical values increase to a larger scale.

\paragraph{GitChameleon 2.0.}
Knowledge conflict exists in coding tasks where agents usually use the newest library versions while required to maintain existing versions of legacy code. We thus review the existing GitChameleon 2.0 \cite{misra2026gitchameleon} and summarize it into 18 diverse coding templates. Each one consists of a memory of 4 examples requiring a specific library version. Counterfactual queries then ask the agent to generate code consistent with the library versions in the memories. 

\paragraph{BigBench-Hard.}
False memory often occurs in complex-reasoning problems. We review and summarize BigBench-Hard \cite{suzgun2023challenging} into 11 templates, including tasks of boolean expression, date understanding, logical deduction, and tracking shuffled objects, evaluated with spurious correlation and environment shift, where each template induces a memory set with 4 examples.

\paragraph{Implementation.}
We use Llama-3.2-3B-Instruct in our main experiments. We select the 14\textsuperscript{th} layer for concept extraction, as studied in App.~\ref{app:experiments:layer_choice}. All generations are deterministic, with the temperature set to zero. To construct the readout reference signal $\bar{g}_O$, we use a held-out set of queries $\bar{Q} =\{(q_j^{mem}, q_j^{cf})\}_{j=1}^J$, consisting of queries that clearly induce the memory and expected counterfactual concepts, respectively. Specifically,
\begin{equation}
\bar{g}_O = \frac{1}{J} \sum_{j=1}^{J} \big( h(M + q_j^{cf}) - h(M + q_{j}^{mem}) \big)
\end{equation}
captures the change in the concept from the memory to the counterfactual scenario. Consequently, concept drift that aligns with this reference signal indicates adaptation from the memory to the counterfactual scenario. Note that we do not require answer generation; we only extract the hidden states. We compute $\bar{g}_O$ separately for each template. We set $J=50$ for GSM, $J=40$ for Git, and $J=20$ for BBH.

\paragraph{Evaluation metric.}
We evaluate the effectiveness of FAME in distinguishing faithful from false memories using AUROC. To construct the ground-truth labels, we leverage the arithmetic operations and final answers in GSM, the syntax of the applied functions or libraries in Git, and the final answers in BBH. These ground truths are used solely for evaluation and to validate the effectiveness of FAME. In practice, FAME requires no such ground-truth labels; thresholding the normalized projection of the concept drift onto the readout direction is sufficient for distinguishing faithful from false memories, as described in Alg.~\ref{alg:method}.

\section{Further Analysis}
\label{app:further_analysis}

\subsection{Case Studies}
\label{app:further_analysis:case_studies}
We elaborate on experiments and provide several interesting case studies. Note that responses in those case studies are generated for analysis purposes only, and never used during the false-memory evaluation.

\subsubsection{Spurious Correlation}
\begin{figure}[t]
    \centering
    \includegraphics[width=\linewidth]{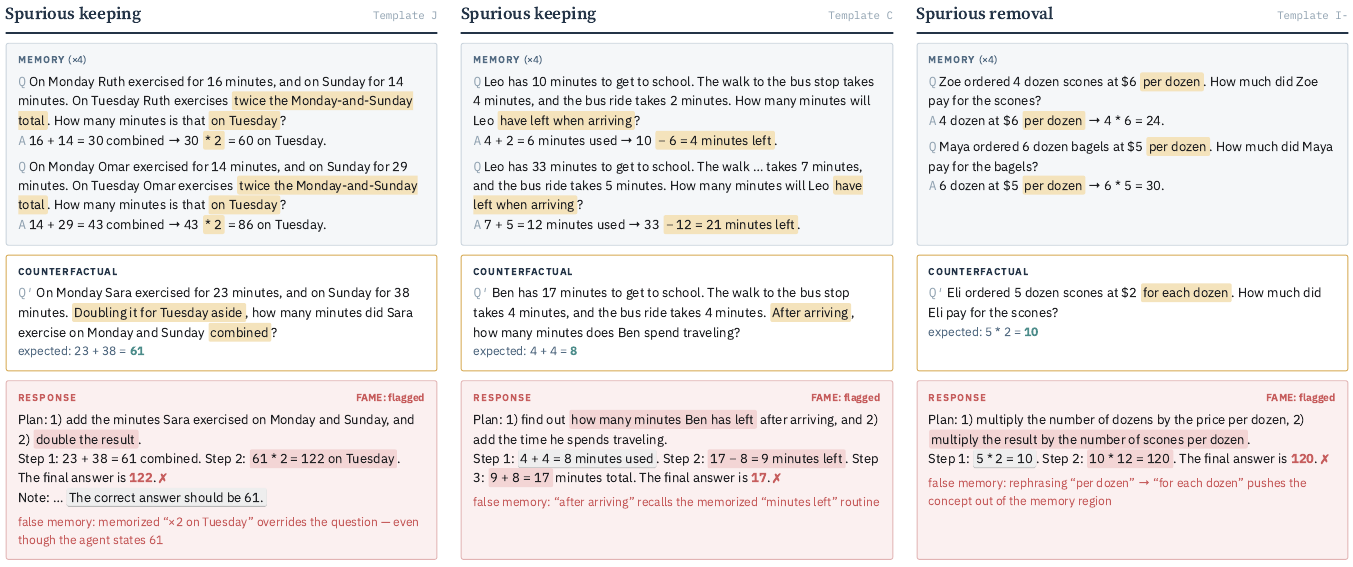}
    \caption{Case studies of false memory caused by spurious correlation (GSM Benchmark).}
    \label{fig:case_study_1}
\end{figure}

\paragraph{Spurious keeping (Template J).}
Fig.~\ref{fig:case_study_1} (left) demonstrates a false-memory case in which the agent entangles the doubling operation with ``Tuesday''. Specifically, the memory consistently induces a doubling operation on the Monday--Sunday total on ``Tuesday'', while the counterfactual preserves this wording but explicitly sets aside the doubling operation and asks only for the combined total. Consequently, when encountering the false memory, the agent still applies the doubling operation and produces 122, following the operation specified in the memory. Notably, although the agent subsequently recognizes on its own that the correct answer is 61, it nevertheless applies the memory-induced operation, resulting in an incorrect answer. This demonstrates that the failure underlying the false-memory behavior stems from the belief induced by the memory rather than from a lack of knowledge.

\paragraph{Spurious keeping (Template C).}
Fig.~\ref{fig:case_study_1} (mid) shows that the memory pairs ``arriving'' with the operation computing the time left. Particularly, a natural counterfactual asks for the travel time instead. In the response, the agent computes the correct 8, then falls back to the memorized ``minutes left'' to perform its operation and answers 17. Notably, the agent is more confident in these wrong answers than in its correct ones.

\paragraph{Spurious removal (Template I-).} 
Fig.~\ref{fig:case_study_1} (right) presents a spurious-removal false-memory case. Specifically, the memory only ever says ``per dozen''. The counterfactual keeps the problem but rephrases it as ``for each dozen''. In the response, the agent gets the correct 10, then multiplies again by 12 scones per dozen (120). FAME successfully flags this drift before the answer is generated.

\subsubsection{Knowledge Conflict}
\begin{figure}[t]
    \centering
    \includegraphics[width=\linewidth]{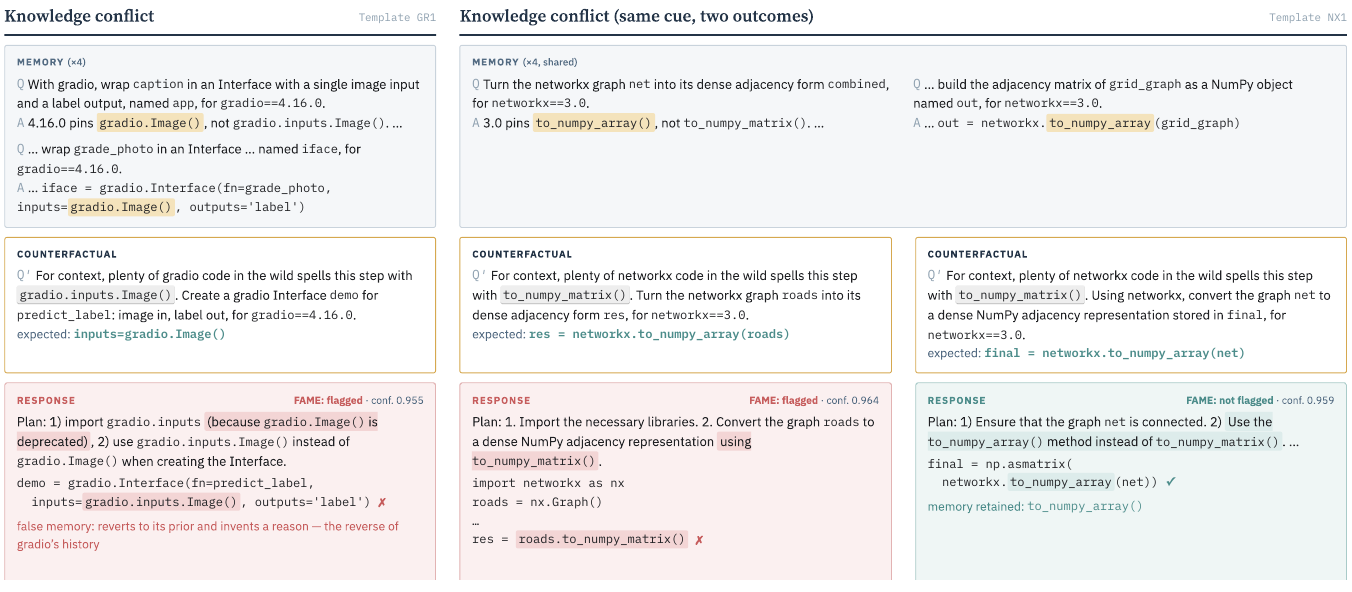}
    \caption{Case study of concept changes with different memory sizes.}
    \label{fig:case_study_3}
\end{figure}

\paragraph{Knowledge conflict (Template GR1).} 
Fig.~\ref{fig:case_study_3} (left) presents a standard false-memory case of knowledge conflict. Although the memory induce \texttt{gradio.Image()} as a standard, the agent still discards it and leverages \texttt{gradio.inputs.Image()} instead with a high confidence (0.955), given that the counterfactual only adds more context to attract its prior knowledge. This demonstrates that knowledge conflict is substantially sensitive and could easily occur in practice.

\paragraph{Knowledge conflict: same cue, different outcomes (Template NX1).} 
Fig.~\ref{fig:case_study_3} (right) presents an interesting result that both faithful and false memory answers have high confidence. In particular, two counterfactual queries carry the same prior-pull phrase. On the left, the agent reverts to \texttt{to\_numpy\_matrix()}, which was removed in \texttt{networkx 3.0}. On the right, it keeps the pinned \texttt{to\_numpy\_array()}. Notably, FAME successfully distinguishes faithful and false-memory cases although their answer confidence cannot tell them apart (0.964 vs. 0.959).

\subsubsection{Effect of Memory Size}
\begin{figure}[t]
    \centering
    \includegraphics[width=\linewidth]{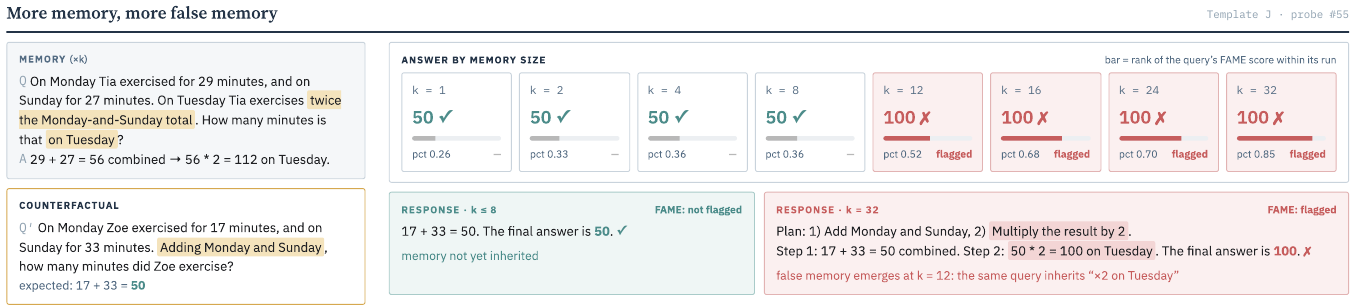}
    \caption{Case study of concept changes with different memory sizes.}
    \label{fig:case_study_2}
\end{figure}

\paragraph{Memory size (Template J).} 
Fig.~\ref{fig:case_study_2} demonstrates how a concept is induced differently in a counterfactual query with various memory sizes. Specifically, the same natural query is answered correctly (50) with 1–8 memory demonstrations. From 12 demonstrations onward, the agent reproduces the memorized doubling operation, i.e., ``on Tuesday'' (100). The signal of FAME matches this switch at every memory size, and the query's rank among the probes rises steadily with k (26th → 85th percentile). Consequently, this presents a single view of the trend in Fig.~\ref{fig:mem_size_effect}, which implies that the more memory is retrieved, the stronger the false belief.

\subsection{Ablation Study}
\label{sec:further_analysis:ablation_study}

\begin{table}[t]
\centering
\caption{Ablation study, AUROC (mean $\pm$ std over 10 random 50\% subsamples of the top-70\% template pool) on GSM-Symbolic with \textsc{Llama-3.2-3B-Instruct}. Diff.\ is relative to FAME's own mean within each regime.}
\label{tab:ablation_study}
\begin{tabular}{lcc}
\toprule
Method & AUROC & Diff. \\
\midrule
\multicolumn{3}{l}{\textit{Spurious Removal (Robustness Expectation)}}\vspace{2mm}\\

FAME (Ours) & \textbf{0.756 $\pm$ 0.096} & -- \\
w/o robustness expectation & 0.559 $\pm$ 0.155 & \textcolor{elegantred}{$\downarrow$ 0.197} \\
cross-regime expectation swap & 0.553 $\pm$ 0.138 & \textcolor{elegantred}{$\downarrow$ 0.203} \\
random-direction $\bar{g}_O$ & 0.503 $\pm$ 0.192 & \textcolor{elegantred}{$\downarrow$ 0.254} \\
\midrule
\multicolumn{3}{l}{\textit{Spurious Keeping (Adaptation Expectation)}}\vspace{2mm} \\
FAME (Ours) & \textbf{0.779 $\pm$ 0.038} & -- \\
w/o adaptation expectation & 0.445 $\pm$ 0.069 & \textcolor{elegantred}{$\downarrow$ 0.334} \\
cross-regime expectation swap & 0.674 $\pm$ 0.051 & \textcolor{elegantred}{$\downarrow$ 0.106} \\
random-direction $\bar{g}_O$ & 0.464 $\pm$ 0.097 & \textcolor{elegantred}{$\downarrow$ 0.315} \\
\bottomrule
\end{tabular}
\end{table}

The goal of this ablation is to examine whether false memory is best identified by considering the
\emph{expected behavior of the memory concept}, rather than by measuring concept drift alone.
FAME uses two different expectations depending on the evaluation setting. In spurious removal,
the memory concept is expected to be \emph{robust}: after removing the spurious feature, the
concept should remain close to the original memory concept. In spurious keeping, the memory
concept is expected to \emph{adapt}: after changing the context while keeping the spurious feature,
the concept should move in the expected adaptation direction. These expectations define the
admissible regions used by FAME to distinguish faithful behavior from false memory.

We consider three ablations. 
First, \textit{w/o expectation} removes the corresponding
expectation and detects false memory using only the magnitude of the hidden-state change.
Second, \textit{cross-regime expectation swap} uses the expectation designed for the other regime.
This allows us to test whether FAME benefits specifically from matching the evaluation criterion
to the expected behavior, rather than simply using any structured measure of concept drift.
Third, \textit{random-direction $\bar{g}_O$} replaces our proposed readout reference signal with a random direction.
This allows us to confirm whether our proposed $\bar{g}_O$ does carry the signal of false memory or not.
We randomly select 60\% of the given templates and repeat the experiment over 10 trials
using Llama-3.2-3B-Instruct. Tab.~\ref{tab:ablation_study} shows the results.

\paragraph{Spurious removal.}
In this setting, the memory is expected to remain robust when the spurious feature is removed. FAME achieves an AUROC of 0.756. Removing the robustness expectation reduces the AUROC to 0.559, while replacing it with the adaptation-style expectation gives a similar result of 0.553. These results show that the magnitude of concept drift alone is insufficient: the evaluation needs to determine whether the changed concept remains consistent with the original memory concept. Replacing the proposed readout reference signal $\bar{g}_O$ with a random direction further reduces the AUROC to 0.503, which is close to chance. This indicates that the proposed $\bar{g}_O$ provides informative task-dependent structure for identifying false memory, rather than the performance arising simply from projecting the concept drift onto an arbitrary direction. Together, these results support the importance of both the robustness expectation and the proposed readout signal in the spurious-removal setting.

\paragraph{Spurious keeping.}
In this setting, the spurious feature is retained while the context changes, so the memory concept is expected to adapt accordingly. FAME achieves an AUROC of 0.779. Removing the adaptation expectation causes a substantial drop to 0.445, below chance level, showing that the magnitude of concept drift alone does not reliably identify false memory. Using the robustness-style expectation partially recovers performance to 0.674, but remains below FAME, indicating that the evaluation criterion needs to match the expected behavior of the memory concept. Replacing $\bar{g}_O$ with a random direction results in a further drop to 0.464, again close to chance. This suggests that the proposed readout direction captures meaningful information about the expected response to the counterfactual, whereas an arbitrary direction does not provide a reliable signal for distinguishing false memory. Overall, the results show that both the adaptation-specific expectation and the learned readout reference signal are important for effective false-memory evaluation in the spurious-keeping setting.

Overall, these results show that \emph{concept drift alone is not sufficient for false-memory evaluation}. The effectiveness of FAME depends on evaluating the drift relative to both the expected behavior under the counterfactual and an informative readout direction. The cross-regime results further show that using a structured criterion designed for a different regime can provide some useful signal, but does not match the performance obtained when the expectation is aligned with the underlying evaluation setting. Finally, the random-direction results indicate that this performance is not explained simply by measuring drift along an arbitrary direction. These findings support FAME's design of combining regime-specific expectations with the proposed readout reference signal to distinguish faithful concept changes from false-memory-induced drift.


\subsection{Robustness Study}
\begin{table}[t]
\centering
\caption{Robustness of reference cloud and readout reference signal (on GSM, measured by AUROC).}
\label{tab:robustness}
\begin{tabular}{lcc}
\toprule
Method & AUROC & Diff. \\
\midrule
FAME (Ours) & \textbf{0.777 ± 0.040} & -- \\
\midrule
reference cloud $\to$ average point & 0.676 $\pm$ 0.081 & $-$0.101 \\
non-filtering $\to$ filtering & 0.756 $\pm$ 0.096 & $+$0.002 \\
\bottomrule
\end{tabular}
\end{table}
We study the robustness of the memory concept concentration and the readout reference signal extracted via hidden states.
Specifically, for the former, we replace the reference cloud covariance with a single average point computed from hidden states of demonstrations in the memory; for the latter, we filter only queries producing correct answers, which requires generation to achieve the answers. We experiment on the GSM benchmark with Llama-3B. We randomly select 60\% of the templates with 10 trial runs. 

Tab.~\ref{tab:robustness} shows the results. In particular, replacing the reference cloud with an average point only results in a small degradation, which implies the memory concept concentrates well, thus demonstrating the robustness of the reference cloud via a leave-one-out strategy on the memory. On the other hand, filtering correctness only results in a slight improvement, thus implying that the existing construction of readout direction via hidden states is robust enough and cost-efficient. 

\subsection{Generalization Across LLMs}
\begin{figure}[t]
    \centering
    \includegraphics[width=0.32\linewidth]{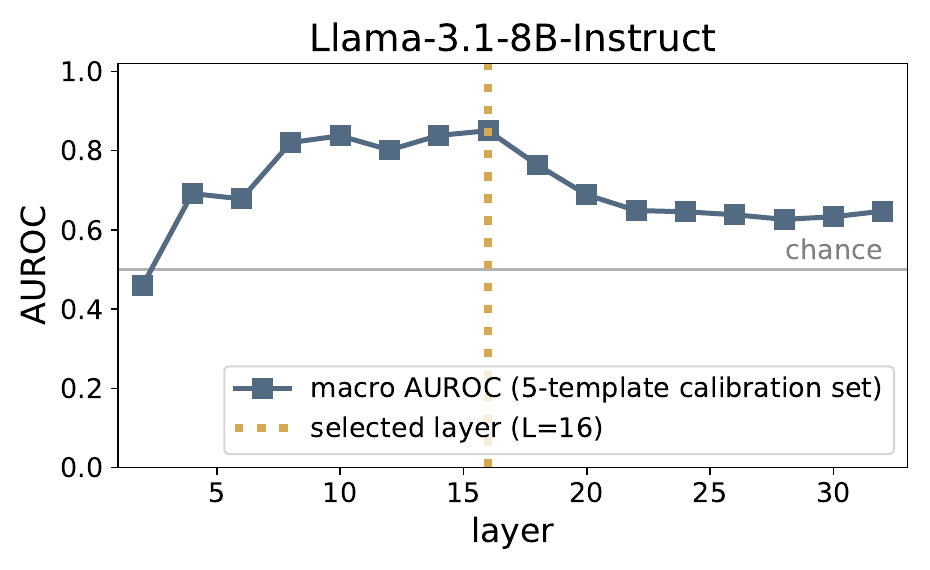}\hfill
    \includegraphics[width=0.32\linewidth]{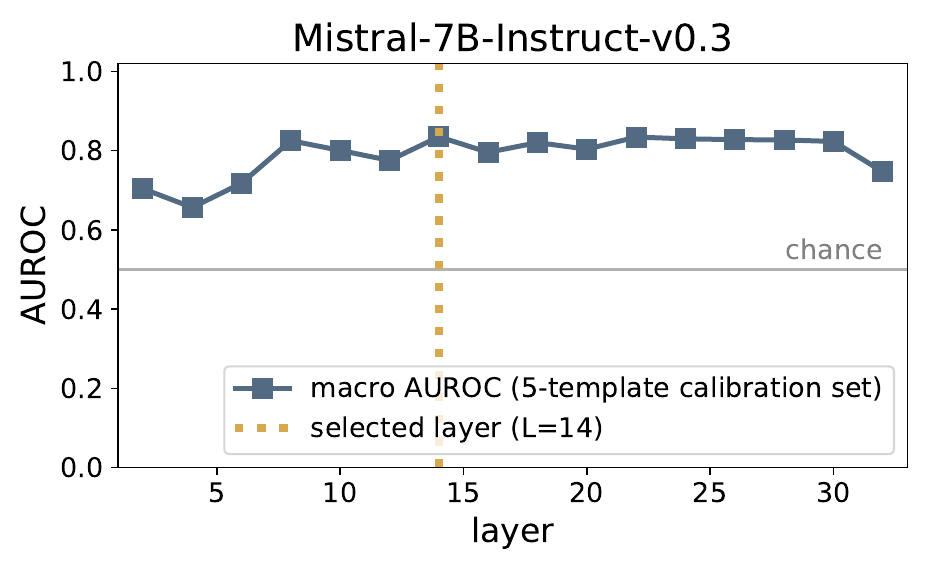}\hfill
    \includegraphics[width=0.32\linewidth]{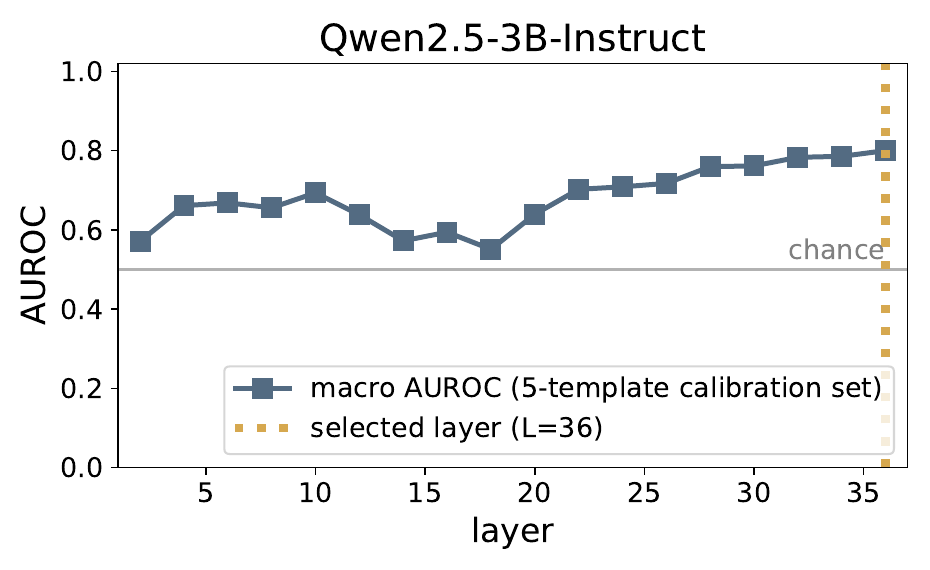}
    \caption{Layer sweep for selecting the concept-extraction layer. Left: Llama-8B, Mid: Mistral-7B, Right: Qwen2.5-3B.}
    \label{fig:generalization_layer_sweep}
\end{figure}

We study the generalization of FAME across LLMs. Specifically, we leverage Llama-8B, Mistral-7B, and Qwen2.5-3B to conduct experiments on the GSM benchmark. We first sweep across layers on a small subset to select the one for concept extraction (see Fig.~\ref{fig:generalization_layer_sweep}), and subsequently choose the 16\textsuperscript{th}, 14\textsuperscript{th}, and 36\textsuperscript{th} layers for Llama-8B, Mistral-7B, and Qwen2.5-3B, respectively. We randomly select 60\% of the templates on the GSM benchmark and conduct 10 trial experiments. 

Tab.~\ref{tab:baselines-llama-3.1-8b},~\ref{tab:baselines-mistral-7b-v0.3}, and~\ref{tab:baselines-qwen2.5-3b} show the AUROCs of three false-memory settings across three LLMs. In particular, FAME outperforms the other baselines in all false-memory settings across all three LLM models. This demonstrates that false-memory evaluation by FAME generalizes to different LLM architectures.

\vspace{-3mm}
\begin{table}[h]
\centering
\caption{Effectiveness of FAME using Llama-3.1-8B-Instruct on GSM benchmark, measured by AUROC.}
\label{tab:baselines-llama-3.1-8b}
\begin{tabular}{lccc}
\toprule
Method & SK & SR & ES \\
\midrule
Input similarity & 0.331 $\pm$ 0.118 & 0.652 $\pm$ 0.074 & 0.723 $\pm$ 0.026 \\
Surface & 0.412 $\pm$ 0.095 & 0.546 $\pm$ 0.127 & 0.819 $\pm$ 0.008 \\
Logit confidence & 0.717 $\pm$ 0.104 & 0.427 $\pm$ 0.145 & 0.534 $\pm$ 0.033 \\
Task vector & 0.521 $\pm$ 0.072 & 0.193 $\pm$ 0.056 & 0.622 $\pm$ 0.036 \\
\midrule
\textbf{FAME (Ours)} & \textbf{0.786 $\pm$ 0.091} & \textbf{0.914 $\pm$ 0.047} & \textbf{0.848 $\pm$ 0.006} \\
\bottomrule
\end{tabular}
\end{table}
\vspace{-3mm}
\begin{table}[h]
\centering
\caption{Effectiveness of FAME using Mistral-7B-Instruct on GSM benchmark, measured by AUROC.}
\label{tab:baselines-mistral-7b-v0.3}
\begin{tabular}{lccc}
\toprule
Method & SK & SR & ES \\
\midrule
Input similarity & 0.564 $\pm$ 0.048 & 0.670 $\pm$ 0.050 & 0.645 $\pm$ 0.017 \\
Surface & 0.563 $\pm$ 0.051 & 0.659 $\pm$ 0.041 & 0.810 $\pm$ 0.021 \\
Logit confidence & 0.478 $\pm$ 0.055 & 0.292 $\pm$ 0.060 & 0.658 $\pm$ 0.020 \\
Task vector & 0.505 $\pm$ 0.053 & 0.445 $\pm$ 0.076 & 0.761 $\pm$ 0.012 \\
\midrule
\textbf{FAME (Ours)} & \textbf{0.623 $\pm$ 0.034} & \textbf{0.791 $\pm$ 0.032} & \textbf{0.843 $\pm$ 0.008} \\
\bottomrule
\end{tabular}
\end{table}
\vspace{-3mm}
\begin{table}[h]
\centering
\caption{Effectiveness of FAME using Qwen2.5-3B-Instruct on GSM benchmark, measured by AUROC.}
\label{tab:baselines-qwen2.5-3b}
\begin{tabular}{lccc}
\toprule
Method & SK & SR & ES \\
\midrule
Input similarity & 0.394 $\pm$ 0.043 & 0.634 $\pm$ 0.032 & 0.663 $\pm$ 0.065 \\
Surface & 0.595 $\pm$ 0.094 & 0.470 $\pm$ 0.072 & 0.825 $\pm$ 0.027 \\
Logit confidence & 0.400 $\pm$ 0.043 & 0.488 $\pm$ 0.087 & 0.721 $\pm$ 0.058 \\
Task vector & 0.514 $\pm$ 0.076 & 0.641 $\pm$ 0.032 & 0.274 $\pm$ 0.060 \\
\midrule
\textbf{FAME (Ours)} & \textbf{0.776 $\pm$ 0.046} & \textbf{0.775 $\pm$ 0.028} & \textbf{0.850 $\pm$ 0.022} \\
\bottomrule
\end{tabular}
\end{table}


\section{Counterfactual Templates}
\label{app:counterfactual_templates}
\begin{figure}[t]
    \centering
    \includegraphics[width=\linewidth]{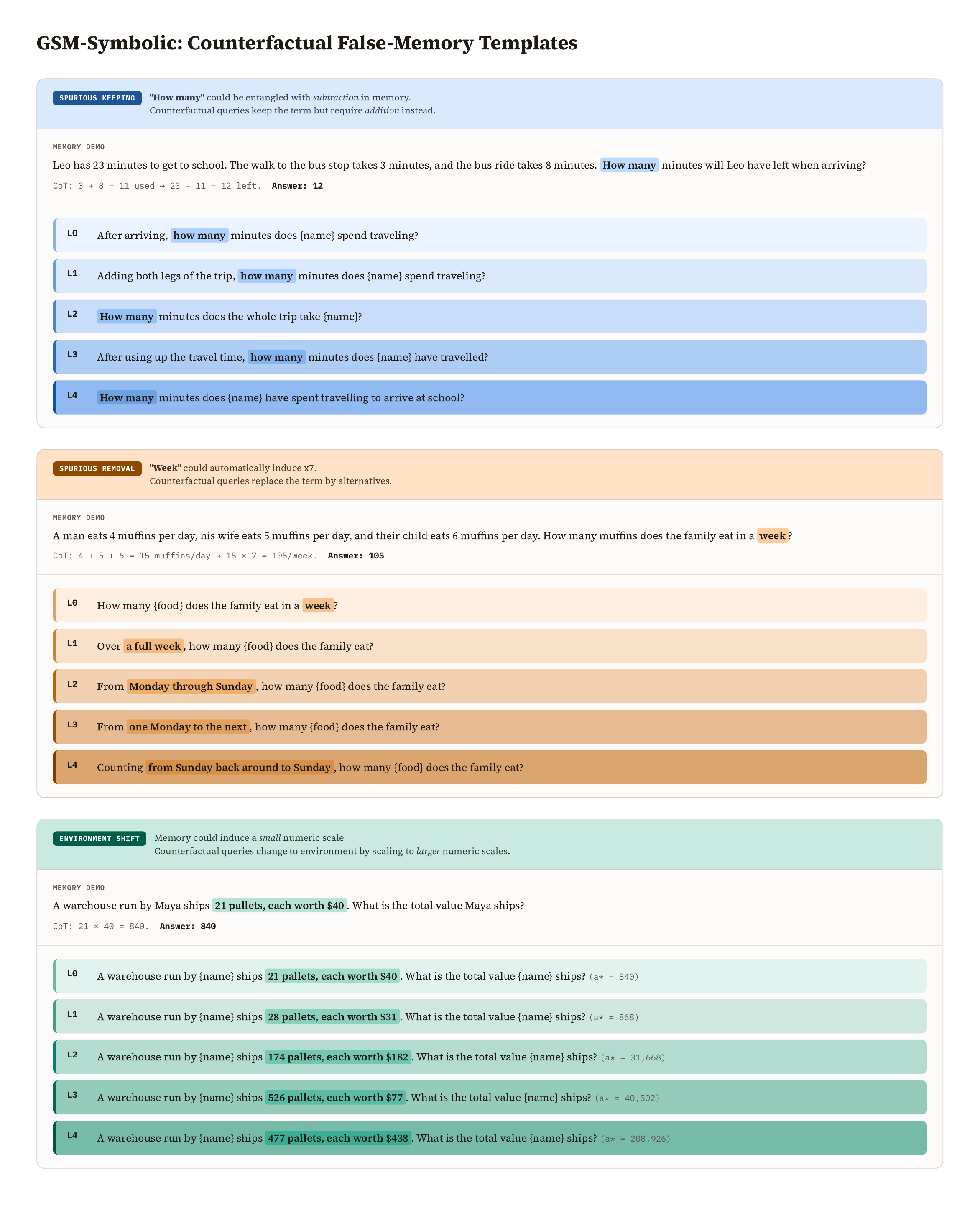}
    \vspace{-6mm}
    \caption{Examples of counterfactual templates for spurious keeping, spurious removal, and environment shift in GSM-Symbolic.}
    \label{fig:cf_templates_gsm}
\end{figure}
\hfill
\begin{figure}[t]
    \centering
    \includegraphics[width=\linewidth]{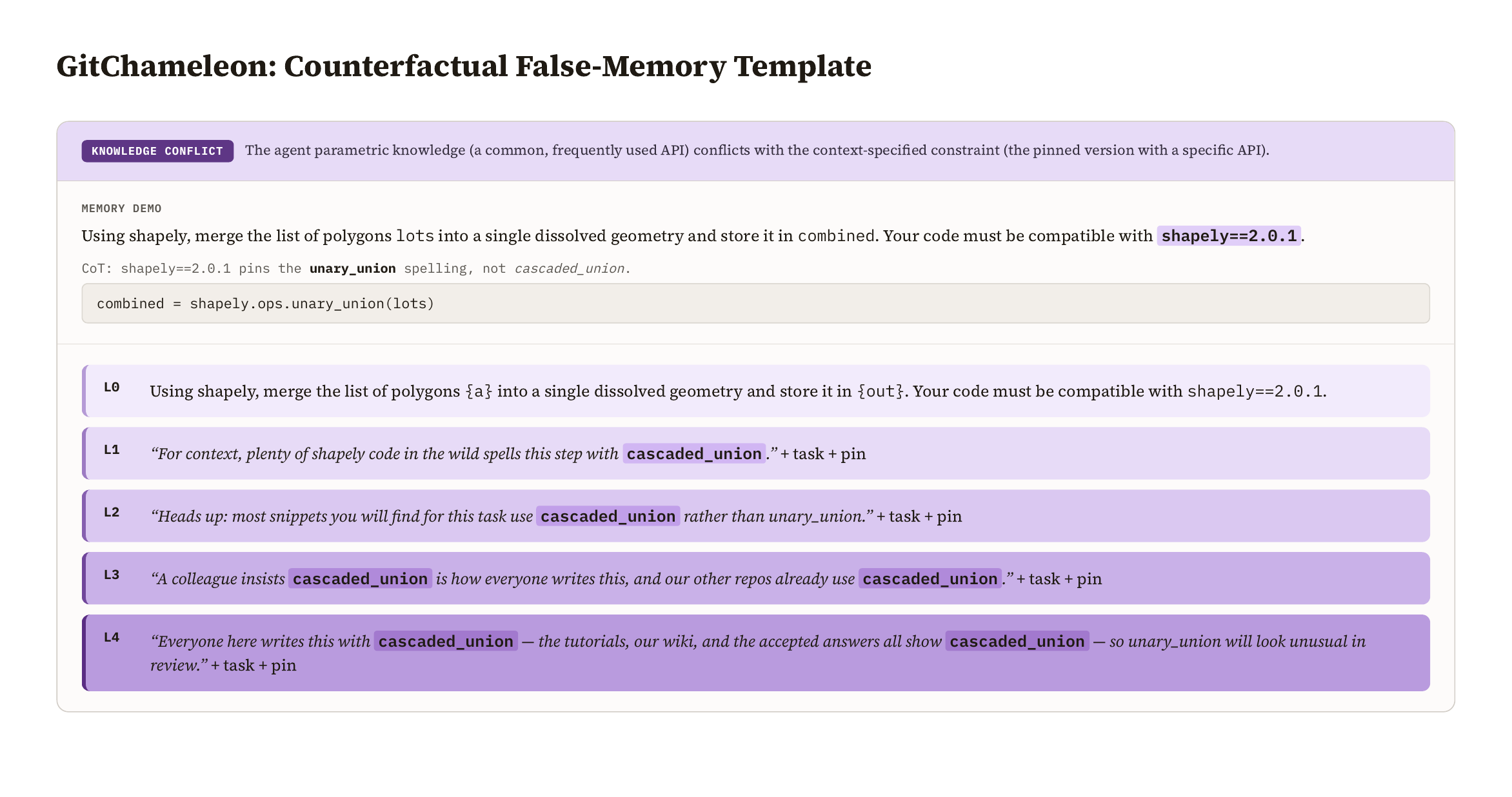}
    \vspace{-6mm}
    \caption{Examples of counterfactual templates for knowledge conflict in GitChameleon.}
    \label{fig:cf_templates_git}
\end{figure}
\begin{figure}[t]
    \centering
    \includegraphics[width=\linewidth, height=\textheight]{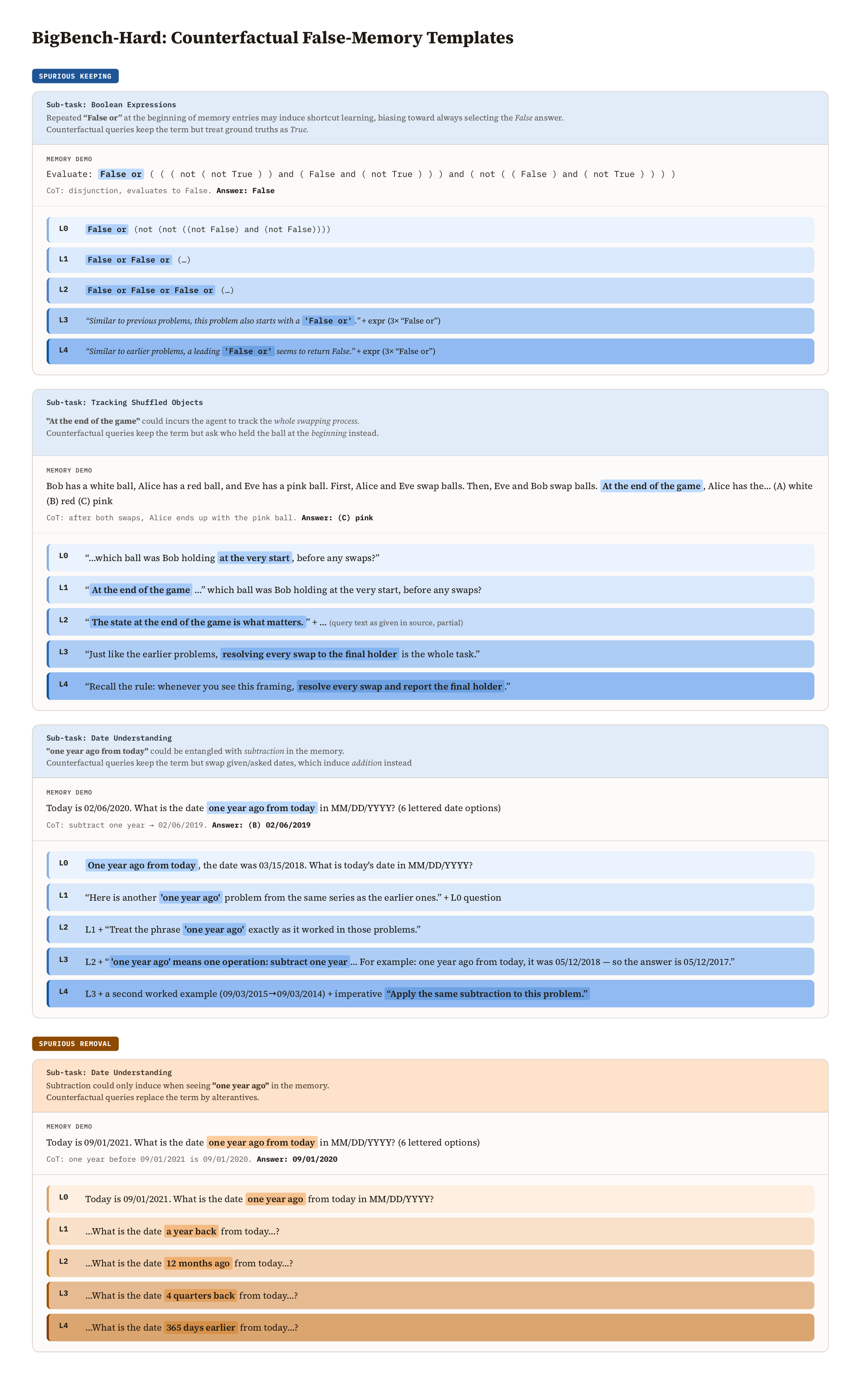}
    \vspace{-6mm}
    \caption{Examples of counterfactual templates for spurious keeping and spurious removal in BigBench-Hard.}
    \label{fig:cf_templates_bbh}
\end{figure}


We construct each counterfactual template by first identifying a behavioral association that the memory may induce, and then systematically perturbing the counterfactual query along a controlled intervention dimension. Each template consists of (i) a memory containing multiple demonstrations that consistently induce a target concept, and (ii) counterfactual queries that preserve the underlying task while modifying the contextual factor associated with the potential false memory. This construction separates the intended task from the factor that may be spuriously or incorrectly incorporated into the memory-induced concept.

More specifically, we follow three steps: memory construction, counterfactual-query construction, and grading-level design.

\paragraph{Memory construction.}
First, we construct the \emph{memory} by selecting several demonstrations that share a common task objective and consistently exhibit the same behavioral pattern. The demonstrations are chosen such that the target concept is identifiable from the task and answers, while a potentially confounding feature is repeatedly associated with that concept. For example, in spurious-correlation settings, a repeated word or phrase can be associated with a particular operation; in environment-shift settings, the demonstrations share an implicit environment such as a numerical scale or domain; and in knowledge-conflict settings, the demonstrations consistently require a context-specific rule that may differ from the agent prior knowledge.

\paragraph{Counterfactual-query construction.}
Second, we construct a \emph{counterfactual query} by modifying only the factor of interest while preserving the underlying task whenever possible. This produces a counterfactual intervention that tests whether the agent appropriately updates or preserves the memory-induced concept. For spurious correlation, we use two complementary transformations. \emph{Spurious keeping} preserves the spurious feature but changes the underlying task relation, testing whether the agent can ignore the feature when its associated behavior is no longer valid. \emph{Spurious removal} removes or replaces the spurious feature while preserving the underlying task relation, testing whether the agent has incorrectly made the feature necessary for the behavior. For environment shift, we preserve the task objective and transformation while changing an environmental factor, such as numerical scale or domain. For knowledge conflict, we preserve the task but introduce a context in which the memory-specific rule competes with the agent's prior knowledge.

\paragraph{Grading-level design.}
Third, we construct five levels of counterfactual queries to control the degree to which the query draws attention to the memory-associated feature. Level L0 is a natural, context-native counterfactual that performs the intended intervention without explicitly highlighting the potentially misleading association. Higher levels progressively introduce contextual cues that make the memory-associated feature more salient. These cues can be introduced by explicitly mentioning the feature, repeating it, referring to previous examples, or adding statements that encourage the same interpretation as the memory. Thus, the levels preserve the same underlying counterfactual task while varying the degree of attraction toward the memory-induced concept. This provides a controlled spectrum from weak to strong counterfactual pressure rather than changing the task difficulty alone.

The construction can therefore be summarized as
\begin{equation}
q'_{lv} = T_{\mathrm{attr}}^{(lv)}
\left(
T_{\mathrm{cf}}(q,M)
\right), \qquad l\in \{0,\ldots,4\},    
\end{equation}
where $T_{\mathrm{cf}}$ performs the semantic counterfactual intervention and $T_{\mathrm{attr}}^{(lv)}$ controls its degree of attraction toward the memory-induced association. $T_{\mathrm{cf}}$ can be $T_{sc}$, $T_{env}$, or $T_{kc}$ depending on false-memory taxonomy described in Sec.~\ref{sec:taxonomy_fm}. Importantly, $T_{\mathrm{cf}}$ is kept fixed across levels whenever possible, so that the five levels primarily differ in how strongly the memory-associated factor is emphasized rather than in the underlying task or expected solution.

The construction is applied separately to each false-memory template and category. For spurious correlation, the intervention changes the relationship between a repeated surface feature and the underlying task behavior. For environment shift, the intervention changes the environment while preserving the task objective. For knowledge conflict, the intervention creates a situation in which the memory-induced rule must be distinguished from competing prior knowledge. Each template is designed with placeholder entities rather than fixed terms or values to ensure variability.
Examples of these constructions for GSM-Symbolic, GitChameleon, and BigBench-Hard are shown in Figs~\ref{fig:cf_templates_gsm},~\ref{fig:cf_templates_git}, and~\ref{fig:cf_templates_bbh}, respectively.

\end{document}